\documentclass{article}

\usepackage[preprint]{neurips_2022}

\usepackage[utf8]{inputenc} 
\usepackage[T1]{fontenc}    
\usepackage{hyperref}       
\usepackage{url}            
\usepackage{booktabs}       
\usepackage{amsfonts}       
\usepackage{nicefrac}       
\usepackage{microtype}      
\usepackage{xcolor}         
\usepackage{bm}
\usepackage{mathtools}
 \usepackage{graphicx}         
\usepackage{hyperref}
\usepackage{natbib}
\usepackage{amsthm}
\usepackage{amssymb}
\usepackage{amsfonts}
\usepackage{amsmath}
\usepackage{xifthen}
\usepackage{xparse}
\usepackage{dsfont}
\usepackage{xcolor}
\usepackage[linesnumbered,ruled,vlined]{algorithm2e} 
\SetKw{KwInit}{Initialize}
\SetKwProg{KwUpdate}{determine}{}{}
\SetKwProg{KwSkip}{skipping}{}{}

\usepackage{Definitions}

\newtheorem{theorem}{Theorem}

\newtheorem{lemma}[theorem]{Lemma}

\newtheorem{corollary}[theorem]{Corollary}
\newtheorem{claim}[theorem]{Claim}
\newtheorem{fact}[theorem]{Fact}

\newtheorem{remk}[theorem]{Remark}
\newtheorem{exmp}[theorem]{Example}

\def\FullBox{\hbox{\vrule width 8pt height 8pt depth 0pt}}

\def\qed{\ifmmode\qquad\FullBox\else{\unskip\nobreak\hfil
\penalty50\hskip1em\null\nobreak\hfil\FullBox
\parfillskip=0pt\finalhyphendemerits=0\endgraf}\fi}

\def\qedsketch{\ifmmode\Box\else{\unskip\nobreak\hfil
\penalty50\hskip1em\null\nobreak\hfil$\Box$
\parfillskip=0pt\finalhyphendemerits=0\endgraf}\fi}

\newcommand{\R}{{\mathbb R}} 

\newcommand{\lr}[1]{\left (#1\right)} 
\newcommand{\lrs}[1]{\left [#1 \right]} 
\newcommand{\lrc}[1]{\left \{#1\right\}} 
\newcommand{\ip}[1]{\left \langle #1 \right \rangle}

\let\P\undefined
\NewDocumentCommand{\P}{o}{\mathbb P{\IfValueT{#1}{\lr{#1}}}}

\NewDocumentCommand{\E}{o}{\mathbb E\IfValueT{#1}{\lrs{#1}}}

\DeclareMathOperator{\V}{Var}
\NewDocumentCommand{\Var}{o}{\V\IfValueT{#1}{\lrs{#1}}}

\NewDocumentCommand{\1}{o}{\mathds 1{\IfValueT{#1}{\lr{#1}}}}

\newcommand{\cA}{\mathcal A}
\newcommand{\cI}{\mathcal I}

\newcommand{\innerprod}[2]{\langle #1, #2\rangle}
\newcommand{\Reg}{\overline{Reg}_T}

\newcommand{\Regdrift}{\overline{Reg}_T^{drift}}

\title{Formatting Instructions For NeurIPS 2022}

\author{%
  Saeed Masoudian\\
  University of Copenhagen\\
  \texttt{saeed.masoudian@di.ku.dk} \\
   \And
   Julian Zimmert \\
   Google Research \\
   \texttt{zimmert@google.com} \\
   \AND
   Yevgeny Seldin\\
  University of Copenhagen\\
  \texttt{seldin@di.ku.dk} \\
}

\title{A Best-of-Both-Worlds Algorithm for Bandits with Delayed Feedback}
\date{\today}

\begin{document}

\maketitle

\begin{abstract}
  We present a modified tuning of the algorithm of \citet{zimmert2020} for adversarial multiarmed bandits with delayed feedback, which in addition to the minimax optimal adversarial regret guarantee shown by \citeauthor{zimmert2020} simultaneously achieves a near-optimal regret guarantee in the stochastic setting with fixed delays. Specifically, the adversarial regret guarantee is $\mathcal{O}(\sqrt{TK} + \sqrt{dT\log K})$, where $T$ is the time horizon, $K$ is the number of arms, and $d$ is the fixed delay, whereas the stochastic regret guarantee is $\mathcal{O}\left(\sum_{i \neq i^*}(\frac{1}{\Delta_i} \log(T) + \frac{d}{\Delta_{i}\log K}) + d K^{1/3}\log K\right)$, where $\Delta_i$ are the suboptimality gaps. We also present an extension of the algorithm to the case of arbitrary delays, which is based on an oracle knowledge of the maximal delay $d_{max}$ and achieves $\mathcal{O}(\sqrt{TK} + \sqrt{D\log K} + d_{max}K^{1/3} \log K)$ regret in the adversarial regime, where $D$ is the total delay, and $\mathcal{O}\left(\sum_{i \neq i^*}(\frac{1}{\Delta_i} \log(T) + \frac{\sigma_{max}}{\Delta_{i}\log K}) + d_{max}K^{1/3}\log K\right)$ regret in the stochastic regime, where $\sigma_{max}$ is the maximal number of outstanding observations. Finally, we present a lower bound that matches regret upper bound achieved by the skipping technique of \citet{zimmert2020} in the adversarial setting.
\end{abstract}

\section{Introduction}
Delayed feedback is a common challenge in many online learning problems, including multi-armed bandits. The literature studying multi-armed bandit games with delayed feedback builds on prior work on bandit problems with no delays. The researchers have traditionally separated the study of bandit games in stochastic environments \citep{thompson1933,robbins1952,lai1985,auer2002a} and in adversarial environments\citep{auer2002b}. However, in practice the environments are rarely purely stochastic, whereas they may not be fully adversarial either. Furthermore, the exact nature of an environment is not always known in practice. Therefore, in recent years there has been an increasing interest in algorithms that perform well in both regimes with no prior knowledge of the regime \citep{bubeck2012b,seldin2014,auer2016,seldin2017,wei2018}. The quest for best-of-both-worlds algorithms for no-delay setting culminated with the Tsallis-INF algorithm proposed by \citet{zimmert2019}, which achieves the optimal regret bounds in both stochastic and adversarial environments. The algorithm and analysis were further improved by \citet{zimmert2021} and \citet{masoudian2021}, who, in particular, derived improved regret bounds for intermediate regimes between stochastic and adversarial. 

Our goal is to extend best-of-both-worlds results to multi-armed bandits with delayed feedback. So far the literature on multi-armed bandits with delayed feedback has followed the traditional separation into stochastic and advesrarial. In the stochastic regime \citet{joulani2013} showed that if the delays are random (generated i.i.d), then compared to the non-delayed stochastic multi-armed bandit setting, the regret only increases additively by a factor that is proportional to the expected delay. In the adversarial setting \citet{bianchi2016} have first studied the case of uniform delays $d$.  They derived a lower bound $\Omega(\max (\sqrt{KT}, \sqrt{dT \log K}))$ and an almost matching upper bound $\Ocal(\sqrt{KT\log K} + \sqrt{dT \log K})$. \citet{thune2019} and \citet{bistritz2019} extended the results to arbitrary delays, achieving $\Ocal(\sqrt{KT\log K} + \sqrt{D \log K})$ regret bounds based on oracle knowledge of the total delay $D$ and time horizon $T$. \citet{thune2019} also proposed a skipping technique based on advance knowledge of the delays "at action time", which allowed to exclude excessively large delays from $D$. Finally, \citet{zimmert2020} introduced an FTRL algorithm with a hybrid regularizer that achieved $\Ocal(\sqrt{KT} + \sqrt{D \log K})$ regret bound, matching the lower bound in the case of uniform delays and requiring no prior knowledge of $D$ or $T$. The regularizer used by \citeauthor{zimmert2020} was a mix of the negative Tsallis entropy regularizer used in the Tsallis-INF algorithm for bandits and the negative entropy regularizer used in the Hedge algorithm for full information games, mixed with separate learning rates:
\begin{equation}\label{eq:regularizer}
F_t(x) =  -2\eta_t^{-1}\lr{\sum_{i = 1}^K \sqrt{x_i}} + \gamma_t^{-1} \lr{\sum_{i = 1}^K x_i (\log x_i -1)}.
\end{equation}
\citet{zimmert2020} also perfected the skipping technique and achieved a refined regret bound $\Ocal(\sqrt{KT} + \min_{S}(|S| + \sqrt{D_{\bar{S}}\log K}))$, where $S$ is a set of skipped rounds and $D_{\bar S}$ is the total delay in non-skipped rounds. The refined skipping technique requires no advance knowledge of the delays and the key step to eliminating the need in advance knowledge was to base it on counting the outstanding observations rather than the delays. The great advantage of skipping is that a few rounds with excessively large or potentially even infinite delays have a very limited impact on the regret bound. One of our contributions in this paper is a lower bound for the case of non-uniform delays, which matches the refined regret upper bound achieved by skipping.

Even though the hybrid regularizer used by \citet{zimmert2020} was sharing the Tsallis entropy part with their best-of-both-worlds Tsallis-INF algorithm from \citet{zimmert2021}, and even though the adversarial analysis was partly similar to the analysis of the Tsallis-INF algorithm, \citet{zimmert2020} did not succeed in deriving a regret bound for their algorithm in the stochastic setting with delayed feedback and left it as an open problem. 

The stochastic analysis of the Tsallis-INF algorithm is based on the self-bounding technique \citep{zimmert2021}. Application of this technique in the no delay setting is relatively straightforward, but in presence of delays it requires control of the drift of the playing distribution from the moment an action is played to the moment the feedback arrives. \citet{bianchi2016} have bounded the drift of the playing distribution of the EXP3 algorithm in the uniform delays setting with a fixed learning rate. But best-of-both-worlds algorithms require decreasing learning rates \cite{MG19}, which makes the drift control much more challenging. The problem gets even more challenging in the case of arbitrary delays, because it requires drift control over arbitrary long periods of time. 

We apply an FTRL algorithm with the same hybrid regularizer as the one used by \citet{zimmert2020}, but with a different tuning of the learning rates. The new tuning has a minor effect on the adversarial regret bound, but allows us to make progress with the stochastic analysis. For the stochastic analysis we use the self-bounding technique. One of our key contributions is a general lemma that bounds the drift of the playing distribution derived from the time-varying hybrid regularizer over arbitrary delays. Using this lemma we derive near-optimal best-of-both-worlds regret guarantees for the case of fixed delays. But even with the lemma at hand, application of the self-bounding technique in presence of arbitrary delays is still much more challenging than in the no delays or fixed delay setting. Therefore, we resort to introducing an assumption of oracle knowledge of the maximal delay, which limits the maximal period of time over which we need to keep control over the drift. Our contributions are summarized below. To keep the presentation simple we assume uniqueness of the best arm throughout the paper. Tools for eliminating the uniqueness of the best arm assumption were proposed by \citet{ito2021}.
\begin{enumerate}
    \item We show that in the arbitrary delays setting with an oracle knowledge of the maximal delay $d_{max}$ our algorithm achieves $\Ocal(\sqrt{KT} + \sqrt{D\log K} + d_{max}K^{1/3} \log K)$ regret bound in the adversarial regime simultaneously with $\mathcal{O}\left(\sum_{i \neq i^*}(\frac{\log T}{\Delta_i} + \frac{\sigma_{max}}{\Delta_{i}\log K}) + d_{max}K^{1/3}\log K\right)$ regret bound in the stochastic regime, where $\sigma_{max}$ is the maximal number of outstanding observations. We note that $\sigma_{max} \leq d_{max}$, but it may potentially be much smaller. For example, if the first observation has a delay of $T$ and all the remaining observations have zero delay, then $d_{max} = T$, but $\sigma_{max}=1$.
    \item In the case of uniform delays the above bounds simplify to $\Ocal(\sqrt{KT} + \sqrt{dT\log K} + d K^{1/3} \log K)$ in the adversarial case and $\mathcal{O}\left(\sum_{i \neq i^*}(\frac{\log T}{\Delta_i} + \frac{d}{\Delta_{i}\log K}) + dK^{1/3}\log K\right)$ in the stochastic case. For $T \geq dK^{1/3}\log K$ the last term in the adversarial regret bound is dominated by the middle term, which leads to the minimax optimal $\Ocal(\sqrt{KT} + \sqrt{dT\log K})$ adversarial regret. The stochastic regret lower bound is trivially $\Omega(\min\{d, \sum_{i \neq i^*}\frac{\log T}{\Delta_i}\})= \Omega(d + \sum_{i \neq i^*}\frac{\log T}{\Delta_i})$ and, therefore, our stochastic regret upper bound is near-optimal.
    \item We present an $\Omega\lr{\sqrt{KT} + \min_{S}(|S| + \sqrt{D_{\bar{S}}\log K})}$ regret lower bound for adversarial multi-armed bandits with non-uniformly delayed feedback, which matches the refined regret upper bound achieved by the skipping technique of \citet{zimmert2020}.
\end{enumerate}

\section{Problem Setting}
 We study the multi-armed bandit with delays problem, in which at time $t = 1,2,\ldots$ the learner chooses an arm $I_t$ among a set of $K$ arms and instantaneously suffers a loss $\ell_{t,I_t}$ from a loss vector $\ell_t \in [0,1]^K$ generated by the environment, but $\ell_{t,I_t}$ is not observed by the learner immediately. After a delay of $d_t$, at the end of round $t+d_t$, the learner observes the pair $(t, \ell_{t, I_t})$, namely, the loss and the index of the game round the loss is coming from. The sequence of delays $d_1,d_2,\dots$ is selected arbitrarily by the environment. Without loss of generality we can assume that all the outstanding observations are revealed at the end of the game, i.e., $t + d_t \leq T$ for all $t$, where $T$ is the time horizon. We consider two regimes, oblivious adversarial and stochastic. 
The performance of the learner is evaluated using pseudo-regret, which is defined as
\[
\overline{Reg}_T = \EE\left[\sum_{t=1}^{T} \ell_{t,I_t}\right] - \min_{i \in [K]} \EE \left[\sum_{t=1}^{T}  \ell_{t,i} \right] = \EE\left[\sum_{t=1}^{T} \left(\ell_{t,I_t} - \ell_{t,i_T^*}\right) \right],
\]
where $i_T^* \in \argmin_{i \in [K]} {\EE\left[\sum_{t=t}^{T}  \ell_{t,i}\right] } $ is a best arm in hindsight in expectation over the loss generation model and the randomness of the learner. In the oblivious adversarial setting the losses are assumed to be deterministic and the pseudo-regret is equal to the expected regret. 

\paragraph{Additional Notation:}
We use $\Delta^{n}$ to denote the probability simplex over $n+1$ points. The characteristic function of a closed convex set $\Acal$ is denoted by $\Ical_{\Acal}(x)$ and satisfies $\Ical_{\Acal}(x) = 0$ for $x \in \Acal$ and $\Ical_{\Acal}(x) = \infty$ otherwise. The convex conjugate of a function $f: \RR^n \rightarrow \RR$ is defined by $f^*(y) = \sup_{x \in \RR^n} \{\langle x,y \rangle - f(x) \}$. We also use bar to denote that the function domain is restricted to $\Delta^{n}$, e.g., $\bar{f}(x) = \begin{cases}f(x),&\text{if $x \in \Delta^{n}$}\\\infty,&\text{otherwise}\end{cases}$. We denote the indicator function of an event $\Ecal$ by $\1[\Ecal]$ and use $\1_t(i)$ as a shorthand for $\1[I_t = i]$. The probability distribution over arms that is played by the learner at round $t$ is denoted by $x_t \in \Delta^{K-1}$. 

\section{Algorithm}
The algorithm is based on Follow The Regularized Leader (FTRL) algorithm with the hybrid regularizer used by \citet{zimmert2020}, stated in equation \eqref{eq:regularizer}. At each time step $t$ let $\sigma_t = \sum_{s =1}^{t-1} \1(s+ d_s \geq t)$ be the number of the outstanding observations and $\Dcal_t = \sum_{s = 1}^t \sigma_t$ be the cumulative outstanding observations, then the learning rates are defined as 
\begin{align*}
    \eta_{t}^{-1} = \sqrt{t + \eta_{0}},\qquad\qquad\qquad\qquad \gamma_{t}^{-1} = \sqrt{\frac{\sum_{s=1}^t \sigma_s + \gamma_{0}}{\log K}},
\end{align*}
where  $\eta_{0} = 10d_{max} + d_{max}^2/\lr{K^{1/3} \log(K)}^2 $ and $\gamma_{0} = 24^2d_{max}^2 K^{2/3} \log(K)$. The update rule for the distribution over actions played by the learner is
\begin{equation}\label{eq:updaterule}
x_{t} = \nabla \Fbar_t^*(-\Lhat_t^{obs}) = 
\arg\min_{x\in\Delta^{K-1}}\langle\Lhat_t^{obs},x\rangle + F_t(x),
\end{equation}

where $\Lhat_t^{obs} = \sum_{s = 1}^{t-1} \hat{\ell}_s \1(s + d_s < t)$ and $\hat{\ell}_s$ is an importance-weighted estimate of the loss vector $\ell_s$ defined by
\[
\hat \ell_{t,i} = \frac{\ell_{t,i}\1[I_t=i]}{x_{t,i}}.
\]
The algorithm at the beginning of iteration $t$ calculates the cumulative outstanding observations $\Dcal_t$ and uses it to define $\gamma_t$. Next, it uses the  FTRL update rule defined in \eqref{eq:updaterule} to draw action $I_t$. Finally, at the end of round $t$ it receives the delayed observations and update cumulative loss estimation vector accordingly so that $L_{t+1}^{obs} = \sum_{s = 1}^{t} \hat\ell_s \1(s + d_s = t)$. The complete  algorithm is provided in Algorithm \ref{alg:FTRL-delay}. 

\begin{algorithm}
\caption{FTRL with advance tuning for delayed bandit}
\label{alg:FTRL-delay}
\DontPrintSemicolon
\LinesNumberedHidden
\KwIn{Learning rate rule $\eta_t$ and $\gamma_t$ }
\KwInit{$\Dcal_0 = 0$ and $\hat{L}_1^{obs} = \mathbf{0}_K$ (where $\mathbf{0}_K$ is a zero vector in $\RR^K$)}\;
\For{$t= 1,\ldots,n$}{

\KwUpdate{$\gamma_t$}{
$\mbox{Set } \sigma_t = \sum_{s = 1}^{t-1}\1(s + d_s > t)$\;
$\mbox{Update } \Dcal_t = \Dcal_{t-1} + \sigma_t$\:

}
Set $x_t = \arg\min_{x\in\Delta^{K-1}} \langle\Lhat_t^{obs},x\rangle + F_t(x)$\;
Sample $I_t \sim x_t$ \;
\For{$s: s+d_s = t$}{
    Observe $(s,\ell_{s,I_s})$\;
    Construct $\hat\ell_s$ and update $\hat L_t^{obs}$\;
}
}
\end{algorithm}



\section{Best-of-both-worlds Regret Bounds for Algorithm \ref{alg:FTRL-delay}}
In this section we provide  best-of-both-worlds regret bounds for Algorithm \ref{alg:FTRL-delay}. First, in Theorem~\ref{theorem:1} we provide regret bounds for an arbitrary delay setting, where we assume an oracle access to $d_{max}$. Then, in Corollary~\ref{cor:cor2} we specialize the result to a fixed delay setting. 

%
\begin{theorem}\label{theorem:1}
Assume that Algorithm~\ref{alg:FTRL-delay} is given an oracle knowledge of $d_{max}$. Then its pseudo-regret for any sequence of delays and losses satisfies 
\[
\overline{Reg}_T = \mathcal{O}(\sqrt{TK} + \sqrt{D\log K} + d_{max} K^{1/3} \log K).
\]
Furthermore, in the stochastic regime the pseudo-regret additionally satisfies
\[
\overline{Reg}_T = \mathcal{O}\left(\sum_{i \neq i^*}(\frac{1}{\Delta_i} \log(T) + \frac{\sigma_{max}}{\Delta_{i}\log K}) + d_{max}K^{1/3}\log K\right).
\]
\end{theorem}
An overview of the proof is provided in Section \ref{sec:proof} and the complete proof is provided in Appendix \ref{sec:analysistheorem1}.  For fixed delays Theorem~\ref{theorem:1} gives the following corollary.
\begin{corollary}\label{cor:cor2}
 If the delays are fixed and equal to $d$, and $T \geq dK^{1/3}\log K$, then the pseudo-regret of Algorithm~\ref{alg:FTRL-delay} always satisfies
 \[
 \overline{Reg}_T = \Ocal(\sqrt{TK} + \sqrt{dT\log K})
 \]
 and in the stochastic setting it additionally satisfies
 \[
 \overline{Reg}_T = \mathcal{O}\left(\sum_{i \neq i^*}(\frac{1}{\Delta_i} \log(T) + \frac{d}{\Delta_{i}\log K}) + d K^{1/3}\log K\right).
 \]
\end{corollary}
In the adversarial regime with fixed delays $d$, regret lower bound is $\Omega\lr{\sqrt{KT}+\sqrt{dT\log K}}$, whereas in the stochastic regime with fixed delays the regret lower bound is trivially $\Omega(d + \sum_{i \neq i^*}\frac{\log T}{\Delta_i})$. Thus, in the adversarial regime the corollary yields the minimax optimal regret bound and in the  stochastic regime it is near-optimal. More explicitly, it is optimal within a multiplicative factor of $\sum _{i\neq i^*}\frac{1}{\Delta_i\log K} + K^{1/3}\log K$ in front of $d$.

If we fix a total delay budget $D$, then uniform delays $d=D/T$ is a special case, and in this sense Theorem~\ref{theorem:1} is also optimal in the adversarial regime and near-optimal in the stochastic regime, although for non-uniform delays improved regret bounds can potentially be achieved by skipping. We also note that having the dependence on $\sigma_{max}$ in the middle term of the stochastic regret bound in Theorem~\ref{theorem:1} is better than having a dependence on $d_{max}$, since $\sigma_{max} \leq d_{max}$, and in some cases it can be significantly smaller, as shown in the example in the Introduction and quantified in the following lemma. 

\begin{lemma}\label{lemma:sigmamax}
    Let $d_{max}(S) = \max_{s \in S}d_s$, where $S \subseteq \lrc{1, \ldots, T}$ is a subset of rounds. Let $\bar S = \lrc{1,\dots,T}\setminus S$ be the remaining rounds. Then
    \[
    \sigma_{max} \leq \min_{S \subseteq \lrc{1, \ldots, T}}\lrc{|S| + d_{max}(\bar S)}.
    \]
\end{lemma}
A proof of Lemma~\ref{lemma:sigmamax} is provided in Appendix~\ref{sec:appA}.
 
\section{A Sketch of the Proof of Theorem~\ref{theorem:1}}\label{sec:proof}
In this section we provide a sketch of the proof of Theorem~\ref{theorem:1}. We provide a proof sketch for the stochastic bound in Section~\ref{sec:stochasticanalysis}. Afterwards, in Section~\ref{sec:adversaryanalysis} we show how the analysis of \citet{zimmert2020} gives the adversarial bound stated in Theorem~\ref{theorem:1}.
\subsection{Stochastic Bound}\label{sec:stochasticanalysis}
We start by providing a key lemma (Lemma~\ref{lemma:key})
 that controls the drift of the playing distribution derived from the time-varying hybrid regularizer over arbitrary delays. We then introduce the drifted version of pseudo-regret defined in \eqref{eq:driftedregret}, for which we use the key lemma to show that the drifted version of the pseudo-regret is close to the actual one. As a result, it is sufficient to bound the drifted version. The analysis of the drifted pseudo-regret follows by the standard analysis of the FTRL algorithm \citep{lattimore2020} that decomposes the pseudo-regret (drifted pseudo-regret in our case) into stability and penalty terms. Thereafter, we proceed by using Lemma~\ref{lemma:key} again, this time to bound the stability term in order to apply the self-bounding technique \citep{zimmert2019}, which yields logarithmic regret for the stochastic setting. Our key lemma is the following.

\begin{lemma}[The Key Lemma]\label{lemma:key}
 For any $i \in [K]$ and $s, t \in [T]$, where $s \leq t$ and $t  - s \leq d_{max}$, we have the following inequality 
 \[
 x_{t,i} \leq 2 x_{s,i}.
 \]
 \end{lemma}
A detailed proof of the lemma is provided in Appendix~\ref{sec:proofkeylemma}. Below we explain the high level idea behind the proof.
\begin{proof}[Proof sketch]
We know that $x_t = \nabla\bar F_t^*(-\hat L_t^{obs})$ and $x_s = \nabla\bar F_s^*(-\hat L_s^{obs})$, so we introduce $\xtil = \nabla\bar F_s^*(-\hat L_t^{obs})$ as an auxiliary variable to bridge from $x_t$ and $x_{s}$. The analysis consists of two key steps and is based on induction on $(t,s)$.\\
\textbf{Deviation Induced by the Loss Shift:} This step controls the drift when we fix the regularization (more precisely the learning rates) and shift the cumulative loss. We prove the following inequality:
    \[
        \xtil_{i} \leq \frac{3}{2}x_{s,i}.
    \]
    Note that this step uses the induction assumption for $(s,s-d_r)$ for all $r < s: r + d_r = s$.\\
\textbf{Deviation Induced by the Change of Regularizer:}
     In this step we bound the drift when the cumulative loss vector is fixed and we change the regularizer. We show that
     \[
        x_{t,i} \leq \frac{4}{3}\xtil_{i}.
     \]
 Combining these two steps gives us the desired bound. A proof of these steps is in Appendix~\ref{sec:proofkeylemma}.  
\end{proof}
We use Lemma~\ref{lemma:key} to relate the drifted pseudo-regret to the actual pseudo-regret.
Let $A_t = \lrc{s \leq t : s + d_s = t}$ be the set of rounds for which feedback arrives at round $t$. We define the observed loss vector at time $t$ as $\hat\ell_t^{obs} = \sum_{s \in A_t} \hat\ell_s$ and the drifted pseudo-regret as
\begin{equation}\label{eq:driftedregret}
\Regdrift = \EE\lrs{\sum_{t=1}^{T} \lr{\innerprod{x_t}{\hat\ell_{t}^{obs}} - \hat\ell_{t,i_T^*}^{obs}}}.
\end{equation}
We rewrite the drifted regret as
\begin{align*}
    \Regdrift &= \EE\lrs{\sum_{t=1}^{T}\sum_{s \in A_t} \lr{\innerprod{x_t}{\hat\ell_{s}} - \hat\ell_{s,i_T^*}}}\\
    &= \sum_{t=1}^{T}\sum_{s \in A_t}\sum_{i=1}^K \EE[x_{t,i} (\hat\ell_{s,i} - \hat\ell_{s,i_T^*})]\\
    &= \sum_{t=1}^{T}\sum_{s \in A_t} \sum_{i=1}^K \EE[x_{t,i}] \Delta_i = \sum_{t=1}^{T} \sum_{i=1}^K \EE[x_{t+d_t,i}] \Delta_i,
\end{align*}
where when taking the expectation we use the facts that $\hat\ell_{s}$ has no impact on $x_t$ and that the loss estimators are unbiased. Using Lemma~\ref{alg:FTRL-delay} we make a connection between pseudo-regret and the drifted version:
\begin{align*}
\Regdrift = \sum_{t=1}^{T}\sum_{i=1}^K \EE[x_{t+d_t,i}] \Delta_i &\geq \sum_{t=1}^{T-d_{max}} \sum_{i=1}^K \frac{1}{2}\EE[x_{t+d_{max},i}] \Delta_i \\
&= \frac{1}{2}\sum_{t=d_{max}+1}^{T} \sum_{i=1}^K \EE[x_{t,i}] \Delta_i\\
&\geq \frac{1}{2}\sum_{t=1}^{T} \sum_{i=1}^K \EE[x_{t,i}] \Delta_i -  \frac{d_{max}}{2} = \frac{1}{2}\Reg - \frac{d_{max}}{2},
\end{align*}
where the first inequality follows by Lemma~\ref{lemma:key} for $(t+{d_{max}}, t+{d_t})$, and the second inequality uses $\sum_{t=1}^{d_{max}}\EE[x_{t,i}] \Delta_i \leq d_{max}$. As a result, we have $\Reg \leq 2\Regdrift + d_{max}$ and it suffices to upper bound $\Regdrift$. We follow the standard analysis of FTRL, which decomposes the drifted pseudo-regret into \emph{stabiltiy} and \emph{penalty} terms as

\[
   \Regdrift = \EE\lrs{\underbrace{\sum_{t=1}^T {\innerprod{x_t}{\hat\ell_{t}^{obs}} + \bar F_t^*(-\hat{L}^{obs}_{t+1}) - \bar F_t^*(-\hat{L}^{obs}_{t})}}_{stability}} + 
    \EE\lrs{\underbrace{\sum_{t=1}^T \bar F_t^*(-\hat{L}^{obs}_{t}) - \bar F_t^*(-\hat{L}^{obs}_{t+1}) - \ell_{t,i_T^*}}_{penalty}}.
\]

For the penalty term we have the following bound by \citet{abernethy2015}
\[
penalty \leq \sum_{t=2}^{T}\lr{ F_{t-1}(x_{t}) - F_{t}(x_{t})} +  F_T(x^*) - F_1(x_{1}).
\]
By replacing the closed form of the regularizer in this bound and using the facts that $\eta_t^{-1} - \eta_{t-1}^{-1} = \Ocal(\eta_t)$ and $\gamma_t^{-1} - \gamma_{t-1}^{-1} = \Ocal(\sigma_t \gamma_t/\log K)$ we obtain
\begin{align}
penalty \leq \Ocal\lr{\sum_{t=2}^{T} \sum_{i\neq i^*} \eta_{t} x_{t,i}^{\frac{1}{2}} + \sum_{t=2}^{T}\sum_{i=1}^K  \frac{\sigma_t \gamma_t x_{t,i}\log(1/x_{t,i})}{\log K}} + 2\sqrt{\eta_0(K-1)} + \sqrt{\gamma_0 \log K}\label{eq:penaltyarbmain}.
\end{align}
To deal with the stability term we derive Lemma~\ref{lemma:stabilityarb}.
\begin{lemma}[Stability]\label{lemma:stabilityarb}
For any $\alpha_t \leq \gamma_t^{-1}$ we have
\begin{equation*}
stability  \leq \sum_{t=1}^T \sum_{i=1}^K  2f_t^{''}(x_{t,i})^{-1} (\hat\ell_{t,i}^{obs}-\alpha_t)^2.
\end{equation*}
Furthermore, let $a_t = |A_t|$. Then $\alpha_t = \frac{\sum_{j=1}^K  f^{''}(x_{t,j})^{-1} \hat\ell_{t,j}^{obs}}{\sum_{j=1}^K  f^{''}(x_{t,j})^{-1}}$ is a valid $\alpha$, for which we have
\begin{equation}\label{eq:stability}
\EE[stability] \leq \sum_{t=1}^T \sum_{i \neq i^*} 2\gamma_t (a_t-1)a_t \EE[x_{t,i}]\Delta_i  +  \sum_{t=1}^T \sum_{s \in A_t}  \sum_{i = 1}^K 2\eta_t \EE[x_{t,i}^{3/2} x_{s,i}^{-1} (1 - x_{s,i})].
\end{equation}
\end{lemma}
A proof of the stability lemma is provided in Appendix~\ref{sec:stabilitylemma}.  We use Lemma~\ref{lemma:key} in \eqref{eq:stability} to give bounds $a_t x_{t,i} = \sum_{s \in A_t} x_{t,i} \leq 2\sum_{s \in A_t} x_{s,i}$ and $x_{t,i}^{3/2} x_{s,i}^{-1} (1 - x_{s,i}) \leq 2^{3/2}x_{s,i}^{1/2} (1- x_{s,i})$. Moreover, in order to remove the best arm $i^*$ from the summation in the later bound we use $x_{s,i^*}^{1/2} (1- x_{s,i^*}) \leq \sum_{i\neq i^*} x_{s,i} \leq \sum_{i\neq i^*} x_{s,i}^{1/2}$. These bounds together with the facts that we can change the order of the summations and that each $t$ belongs to exactly one $A_s$, gives us the following stability bound
\begin{equation}\label{eq:stabilityarbmain}
\EE[stability] = \Ocal\lr{\sum_{t=1}^T \sum_{i \neq i^*} \eta_{t} \EE[x_{t,i}^{1/2}] + \sum_{t=1}^T \sum_{i \neq i^*} \gamma_{t+d_t} (a_{t+d_t}-1) \EE[x_{t,i}]\Delta_i}.
\end{equation}

By combining \eqref{eq:stabilityarbmain}, \eqref{eq:penaltyarbmain}, and the fact that $\Reg \leq 2\Regdrift + d_{max}$, we show that there exist constants $a,b,c \geq 0$, such that
\begin{align}\label{eq:combinationarbmain}
 \Reg &\leq \EE\lrs{\underbrace{\sum_{t=1}^T \sum_{i \neq i^*} a \eta_t x_{t,i}^{1/2}}_{A} + \underbrace{\sum_{t=1}^T \sum_{i \neq i^*} b\gamma_{t+d_t} (a_{t+d_t}-1) x_{t,i}\Delta_i}_B + \underbrace{\sum_{t=2}^{T}\sum_{i=1}^K  \frac{c\sigma_t \gamma_t x_{t,i}\log(1/x_{t,i})}{\log K}}_C}\notag\\
    &+  \underbrace{4\sqrt{\eta_0(K-1)} + 2\sqrt{\gamma_0 \log K} + d_{max}}_{D}.
\end{align}

\textbf{Self Bounding Analysis:}
We use the self-bounding technique to write $\Reg = 4\Reg - 3 \Reg$, and then based on \eqref{eq:combinationarbmain} we have 
\begin{equation}\label{eq:mainbound}
    \Reg \leq \EE\lrs{4A - \Reg} + \EE\lrs{4B - \Reg} + \EE\lrs{4C - Reg_T} + 4 D.
\end{equation}
For $D$ we can substitute the values of  $\gamma_0$ and $\eta_0$ and get
\begin{equation}\label{eq:Dbound}
    D = \Ocal(d_{max}(K-1)^{1/3}\log K).
\end{equation}
Upper bounding $A,B$, and $C$ requires separate and elaborate analyses, for which we provide three lemmas, Lemma~\ref{lemma:Abound}, \ref{lemma:Bbound} and \ref{lemma:Cbound}. Proofs of these lemmas are in Appendix~\ref{sec:selfbouningproofs}.

\begin{lemma}[Bound for $4A - Reg_T$]\label{lemma:Abound}
We have the following bound for any $a \geq 0$:
\begin{equation}\label{eq:Abound}
        4A - Reg_T \leq \sum_{i \neq i^*} \frac{16a^2}{\Delta_i}\log (T/\eta_0 + 1).
\end{equation}
\end{lemma}
Lemma~\ref{lemma:Abound} leads to the logarithmic bound in terms of $T$ in our regret. 
\begin{lemma}[Bound for $4B - \Reg$]\label{lemma:Bbound}
Let $a_{max} = \max_{t \in T} a_t$, then we have the following bound
   \begin{equation}\label{eq:Bbound}
       4B - \Reg \leq 32b^2a_{max}\log K.
   \end{equation} 
\end{lemma}
It is evident that $a_{max} \leq d_{max}$, so the bound of Lemma~\ref{lemma:Bbound} leads to  $\Ocal(d_{max} \log K)$ term in the regret bound.
\begin{lemma}[Bound for $4C - \Reg$]\label{lemma:Cbound}
We have the following bound for $4C - \Reg$
   \begin{equation}\label{eq:Cbound}
       4C - \Reg \leq \sum_{i\neq i^*} \frac{128c^2\sigma_{max}}{\Delta_i \log K}.
   \end{equation} 
\end{lemma}
Part of the pseudo-regret bound that corresponds to Lemma~\ref{lemma:Cbound}, comes from the penalty term related to the negative entropy part of the regularizer. In this part, despite the fact that $\sigma_{max}$ can be much smaller than $d_{max}$ (Lemma~\ref{lemma:sigmamax}), the $\sum_{i\neq i^*} \frac{\sigma_{max}}{\Delta_i \log K}$ term could be very large when suboptimal gaps are small. In Appendix~\ref{sec:removingdelta} we show how having an asymmetric learning rate $\gamma_{t,i} \simeq \gamma_t / \sqrt{\Delta_i}$ for negative entropy regularizer allows to remove the factor $\sum_{i\neq i^*} 1/\Delta_i$ in front of $\sigma_{max}$.

Finally, by plugging \eqref{eq:Dbound},\eqref{eq:Abound},\eqref{eq:Bbound},\eqref{eq:Cbound} into \eqref{eq:mainbound} we obtain the desired regret bound.

\subsection{Adversarial Bound}\label{sec:adversaryanalysis}
For the adversarial regime we use the final bound of \citet{zimmert2021}, which holds for any non-increasing learning rates:
\[
\Reg \leq \sum_{t=1}^T \eta_t \sqrt{K} + \sum_{t=1}^T \gamma_t \sigma_t + 2\eta_{T}^{-1}\sqrt{K}  + \gamma_T^{-1}\log K.
\]
It suffices to substitute the values of the learning rates and use Lemma~\ref{lemma:integralinequality} for function $\frac{1}{\sqrt{x}}$:
\begin{align*}
    \Reg &\leq \sum_{t=1}^T \frac{\sqrt{K}}{\sqrt{t + \eta_0}} + \sum_{t=1}^T \frac{\sigma_t\sqrt{\log K}}{\sqrt{D_t + \gamma_0}} + 2\sqrt{KT + K\eta_0} +   \sqrt{\log(K)D_T + \gamma_0\log(K)}\\
    &= \Ocal\lr{\sqrt{KT} + \sqrt{\log(K)D_T} + d_{max} K^{1/3}\log K}.
\end{align*}




\section{Refined lower bound}
In this section, we prove a tight lower bound for adversarial regret with arbitrary delays.
Both \citet{zimmert2020} and \citet{thune2019} derived improved regret bounds when some delays are extremely large by skipping feedback that takes too long to arrive.
It has been open ever since whether their bounds are tight.
We answer this positively by showing that the regret bound of \citet{zimmert2020} is not improvable without additional assumptions.
We first derive a refined lower bound for full-information games with variable loss ranges,
which might be of independent interest. A proof is provided in Appendix~\ref{sec:lowerboundsproof}.

\begin{theorem}
\label{thm:full-info var range}
Let $L_T\geq L_{T-1}\geq\dots\geq L_1\geq 0$ be a non-increasing sequence of positive reals and assume that there exists a permutation $\rho:[T]\rightarrow[T]$, such that the losses at time $t$ are bounded in $[0,L_{\rho(t)}]^K$.
The minimax regret of the adversarial full-information game is lower bounded by
\begin{align*}
    {Reg}^* = \max\left\{\frac{1}{2}\sum_{t=1}^{\lfloor\log_2(K)\rfloor}L_t,\frac{1}{32}\sqrt{\sum_{t=\lfloor\log_2(K)\rfloor}^T L_t^2\log(K) } \right\}\,.
\end{align*}
\end{theorem}
From here we can directly show a lower bound for the full-information game with variable delays. This implies the same lower bound for bandits, since we have strictly less information available.
\begin{corollary}
\label{thm:full-info lower bound}
Let $(d_t)_{t=1}^T$ be a sequence of non-increasing delays, such that $d_t \leq T+1-t$ and let an oblivious adversary select all loss vectors $(\ell_t)_{t=1}^T$ in $[0,1]^K$ before the start of the game.
The minimax regret of the full-information game is bounded from below by
\begin{align*}
{Reg}^* = \Omega\left(\min_{S\subset[T]}|S|+\sqrt{D_{\bar S}\log(K)}\right)\,,\mbox{ where }D_{\bar S}=\sum_{t\in[T]\setminus S}d_t\,.
\end{align*}
\end{corollary}
\begin{proof}
We divide the time horizon greedily into $M$ buckets, such that the actions for all timesteps inside a bucket have to be chosen before the first feedback from any timestep inside the bucket is received. In other words, let bucket $B_m = \{b_m,\dots,b_{m+1}-1\}$, then $\forall t\in B_m:t+d_t> b_{m+1}-1$, while $\exists t\in B_m: t+d_t=b_{m+1}$.
This division of buckets has the following properties:
\begin{itemize}
    \item[{\it(i)}] monotonically decreasing sizes: $|B_1|\geq |B_2|\geq\dots\geq|B_M|$.
    \item[{\it (ii)}] upper bound on the sum of delays: $\forall m\in[M-1]: |B_m|^2\geq \sum_{t\in B_{m+1}}d_t$.
\end{itemize}
Both properties follow directly from the non-decreasing nature of the delays.
\begin{align*}
    &|B_m| = b_{m+1}-b_m \leq b_m + d_{b_m}-b_m = d_{b_m}\\
    &|B_m| = \min_{t\in B_m}\{ d_t + t-b_m\} \geq d_{b_{m+1}-1}+\min_{t\in B_m}\{  t-b_m\}=d_{b_{m+1}-1}\,.
\end{align*}
Hence 
\begin{align*}
&|B_m|\geq d_{b_{m+1}-1}\geq d_{b_{m+1}}\geq |B_{m+1}|\,,\\
&\textstyle\sum_{t\in B_{m+1}}d_t \leq |B_{m+1}|\cdot d_{b_{m+1}}\leq |B_{m+1}|\cdot |B_m|\leq |B_m|^2\,.
\end{align*}
Set $S'= \bigcup_{m=1}^{\lfloor\log_2(K)\rfloor}B_m$ and
let the adversary set all losses within a bucket to the same value, then the game reduces to a full information game over $M$ rounds with loss ranges $|B_1|,|B_2|,\dots,|B_M|$.
Applying Theorem~\ref{thm:full-info var range} yields
\begin{align*}
    {Reg}^* &= \max\left\{\frac{1}{2}\sum_{m=1}^{\lfloor\log_2(K)\rfloor}|B_m|,\frac{1}{32}\sqrt{\sum_{m=\lfloor\log_2(K)\rfloor}^M |B_m|^2\log(K) } \right\}\\
    &\geq \max\left\{\frac{1}{2}|S'|,\frac{1}{32}\sqrt{\sum_{t\in\bar S'} d_t\log(K) } \right\} = \Omega\left(\min_{S\subset[T]}|S|+\sqrt{\sum_{t\in\bar S} d_t\log(K) }\right)\,.
\end{align*}
\end{proof}

\section{Discussion}

We have presented a best-of-both-worlds analysis of a slightly modified version of the algorithm of \citet{zimmert2020} for bandits with delayed feedback. The key novelty of our analysis is the control of the drift of the playing distribution over arbitrary, but bounded, time intervals when the learning rate is changing over time. This control is necessary for best-of-both-worlds guarantees, but it is much more challenging than the drift control over fixed time intervals with fixed learning rate that appeared in prior work. 

We also presented an adversarial regret lower bound matching the skipping-based refined regret upper bound of \citet{zimmert2020} within constants.

Our work leads to several exciting open questions. The main one is whether skipping can be used to eliminate the need in oracle knowledge of $d_{max}$. If possible, this would remedy the deterioration of the adversarial bound by the additive factor of $d_{max}$, because the skipping threshold would be dominated by $\sqrt{D_{\bar S} \log K}$. Another open question is whether $\frac{\sigma_{max}}{\Delta_i}$ term can be eliminated from the stochastic bound. Yet another open question is whether the $d_{max}$ factor in the stochastic bound can be reduced to $\sigma_{max}$ and whether the multiplicative terms dependent on $K$ can be eliminated.

\begin{ack}
This project has received funding from European Union's Horizon 2020 research and innovation programme under the Marie Skłodowska-Curie grant agreement No 801199. YS acknowledges partial support by the Independent Research Fund Denmark, grant number 9040-00361B.
\end{ack}

 \small

 \bibliographystyle{plainnat}
\bibliography{references.bib}

\newpage

\appendix

\section{Proof of Lemmas for Algorithm~\ref{alg:FTRL-delay}}\label{sec:appA}
\subsection{Proof of Lemma~\ref{lemma:sigmamax}}
\begin{proof}
Let $S$ be an arbitrary subset of rounds and suppose that we are at round $t$. If we take the rounds $S$ out of our calculations for $\sigma_{t}$, then the outstanding observations with respect to the remaining rounds, denoted by $\bar S$, is $d_{max}(\bar S)$. On the other hand, if the skipped rounds $S$ all have infinite delays then the outstanding observations at round $t$ w.r.t. the rounds in $S$ is at most $|S|$. So the actual outstanding observations is at most $d_{max}(\bar S) + |S|$.
 \end{proof}


\subsection{Proof of Self Bounding
Lemmas}\label{sec:selfbouningproofs}
First we provide some auxiliary lemmas which are helpful for proving the self-bounding lemmas. 
\begin{lemma}[Integral inequality: Lemma 4.13 of \citet{Orabona2019}]\label{lemma:integralinequality}
 Let $g(x)$ be positive, and nonincreasing function then for all sequence of $\{z_n\}_{n \in [N]}$ we have
 \[
 \sum_{n=1}^{N} z_n g\lr{\sum_{i = 0}^n z_i} \leq \int_{z_0}^{\sum_{i = 0}^N z_i}  g(x) dx.
 \]

\end{lemma}

\begin{lemma}\label{lemma:observationarriving}
        Let $\sigma_t$, and $a_t$ be the number of outstanding observations and arriving observations at time $t$, respectively, then the following inequality holds for all $t$
        \[
        \sum_{s = 1}^t \sigma_s \geq \sum_{s=1}^t \frac{a_s^2 - a_s}{2} 
        \]
\begin{proof}
Note that $A_s = \lrc{r : r + d_r = t}$. Then we define $D_s = \lrc{d_r: r \in A_s}$ be the set of the delays of the rounds that arrive at round $s$. $D_s$ must have $a_s = |A_s|$ different number of elements because $\forall r \in A_s: r + d_r = s$. As a result we have
\[
\sum_{r \in A_s} d_r \geq 0 + 1 + \ldots + a_s-1 = a_s(a_s-1)/2.
\]
This gives us the following inequality
\begin{align*}
    \sum_{s = 1}^t \frac{a_s^2-a_s}{2} &= \sum_{s = 1}^t\sum_{r \in A_s} d_r\\
    &\leq \sum_{r \in \bigcup\limits_{s=1}^{t} A_s} d_r = \sum_{r + d_r \leq t} d_r.
\end{align*}
On the other hand for each round $r$, where $r + d_r \leq t$, we count it exactly $d_r$ number of times as an outstanding observation in $\sum_{s = 1}^t \sigma_s$. Therefore, we have  
$\sum_{s = 1}^t \sigma_s \geq \sum_{r + d_r \leq t} d_r$
that together with the above inequality complete the proof.
\end{proof}
\end{lemma}

\subsubsection{Proof of Lemma~\ref{lemma:Abound}}

\begin{proof}
Here we give bound for $4A - Reg_T$.
    \begin{align}
        4A - Reg_T &= \sum_{t =1}^T\sum_{i \neq i^*} \lr{\frac{4a x_{t,i}^{\frac{1}{2}}}{\sqrt{t+\eta_0}}- x_{t,i}\Delta_i}\notag\\
        &\leq \sum_{t =1}^T \sum_{i \neq i^*} \frac{16a^2}{(t + \eta_0) \Delta_i} \leq \sum_{i \neq i^*} \frac{16a^2}{\Delta_i}\log (T/\eta_0 + 1)\label{eq:selfAarb},
    \end{align}
    where the first inequality uses $\forall x,y\geq 0:  x + y \geq 2\sqrt{xy} \Rightarrow 2\sqrt{xy} - y \leq x$ so called AM-GM, and the second inequality follows the integral inequality for logarithm function which is $\sum_{t=1}^T 1/(t+\eta_0) \leq \log(T+\eta_0) - \log(\eta_0)$.
\end{proof}

\subsubsection{Proof of Lemma~\ref{lemma:Bbound}}
\begin{proof}
We have  
\[
4B - Reg_T = \sum_{t=1}^T \sum_{i \neq i^*}  x_{t,i}\Delta_i\lr{4b (a_{t+d_t}-1)\gamma_{t+d_t}  - 1}.
\]
We define $T_0$ be the first $t$ where $\gamma_t^{-1} \geq 4b (a_{max}-1)$, where $a_{max} = \max_{s \in [T]}\{a_s\}$. So in the  summation over time, the rounds with $t+d_t \geq T_0$ have a negative effect as $b (a_{t+d_t}-1) \gamma_{t+d_t} - 1 \leq \frac{4b(a_{t+d_t}-1)}{4b(a_{max}-1)} - 1 \leq 0$. Therefore, we skip them in the summation 
 
 \begin{align}
    4B - Reg_T   &\leq    \sum_{t + d_t < T_0} \sum_{i \neq i^*}x_{t,i}\Delta_i\lr{4b (a_{t+d_t}-1)  \gamma_{t+d_t}  -
    1}\notag\\
    &\leq \sum_{t + d_t < T_0} 4b (a_{t+d_t}-1)  \gamma_{t+d_t} =  \sum_{t=1}^{T_0-1} \sum_{s + d_s = t} 4b (a_{t}-1)  \gamma_{t+d_t} = \sum_{t=1}^{T_0-1} 4b a_t(a_{t}-1)  \gamma_{t+d_t}, \label{eq:selfboundingB12}
\end{align}
where the second inequality holds because $\sum_{i\neq i^*} x_{t,i} \Delta_i \leq 1$ and $a_{t+d_t} \geq 1$.
For simplicity of notation, we denote $\atil_{t} = a_t(a_t-1)/2$ for which Lemma~\ref{lemma:observationarriving} gives us $\sum_{s = 1}^t \atil_{t} \leq \sum_{s = 1}^t \sigma_s$. Therefore using this inequality we have 
    \begin{align}
    \sum_{t=1}^{T_0-1} 4b a_t(a_{t}-1)  \gamma_{t+d_t} &\leq \sum_{t=1}^{T_0-1} \frac{8b\sqrt{\log K} \atil_{t}}{\sqrt{\sum_{s = 1}^t \atil_{t}}}\notag\\
    &\leq 8b \sqrt{\log K\sum_{t = 1}^{T_0-1} \atil_{t}} \leq  8b \sqrt{\log K\sum_{t = 1}^{T_0-1} \sigma_{t}} = 8b\log K \gamma_{T_0-1}^{-1}, \label{eq:selfboundingB22}
    \end{align}

where the second inequality uses integral inequality Lemma \ref{lemma:integralinequality} for $g(x) = \frac{1}{\sqrt{x}}$. Moreover, from the way we chose $T_0$ we have $\gamma_{T_0-1}^{-1}\leq 4b(a_{max}-1)$. Combining this with \eqref{eq:selfboundingB12} and \eqref{eq:selfboundingB22} gives us  $B - \Reg \leq 32b^2a_{max}\log K$.
\end{proof}

\subsubsection{Proof of Lemma~\ref{lemma:Cbound}}
\begin{proof}
First we remove $i^*$ from $C$ by using the following inequality
    \[
    -x_{t,i^*} \log(x_{t,i^*}) \leq (1 - x_{t,i^*}) = \sum_{i \neq i^*} x_{t,i}.
    \]
    This is derived by the fact that $z \log(z) + 1-z $ is decreasing in $z \in [0,1]$ and the minimum value is zero, so consequently it is positive elsewhere. Now using this inequality we can take $i^*$ out from $C$
    \[
    \sum_{t=2}^{T}\sum_{i=1}^K  \frac{-4c\sigma_t x_{t,i}\log(x_{t,i})}{\sqrt{(S_t + \gamma_0)\log K}} \leq 4\underbrace{\sum_{t=1}^{T} \sum_{i\neq i^*}  \frac{- c\sigma_t x_{t,i}\log(x_{t,i})}{\sqrt{(S_t + \gamma_0) \log K}}}_{C_1} + 4\underbrace{\sum_{t=1}^{T}\sum_{i\neq i^*}  \frac{c\sigma_t x_{t,i}}{\sqrt{(S_t + \gamma_0) \log K}}}_{C_2},
    \]
    where $S_t = \sum_{s=1}^t \sigma_s$.
    We will break the self bounding expression, $C - Reg_T$, into $4\lr{C_1 - \alpha Reg_T} + 4\lr{C_2 - \beta Reg_T}$, where $\alpha + \beta = 1/4$. \\
    \textbf{Self bounding for C2}\\
    Let $\sigma_{max} = \max_{t \in [T]}\{\sigma_t\}$ and define $T_i$ be the first $t$ where $S_t + \gamma_0 \geq \frac{c^2\sigma_{max}^2}{\beta^2\Delta_i^2\log K}$. So for all $t \geq 0$ we have 
    \[
    \frac{c\sigma_t x_{t,i}}{\sqrt{(S_t + \gamma_0) \log K}} - \beta x_{t,i}\Delta_i \leq 0.
    \]
    Therefore we can skip the rounds with negative effect in the summation over time and get
    \begin{align}
        C_2 - \beta Reg_T &\leq \beta \sum_{i\neq i^*} \sum_{t=1}^{T_i-1} x_{t,i}\lr{\frac{c\sigma_t}{\beta \sqrt{(S_t + \gamma_0) \log K}} - \Delta_i}\notag\\
        &\leq \sum_{i\neq i^*} \sum_{t=1}^{T_i-1} \frac{c\sigma_t}{\sqrt{(S_t + \gamma_0) \log K}} \notag\\
        &\leq \sum_{i\neq i^*}  \frac{2c(\sqrt{S_{T_i-1} + \gamma_0} -\sqrt{\gamma_0})  }{\sqrt{\log K}}\notag\\
        &\leq \sum_{i\neq i^*} \frac{2c^2\sigma_{max}}{\beta \Delta_i\log K},\label{eq:c2arb}
    \end{align}
    where the third inequality uses Lemma \ref{lemma:integralinequality} for $g(x)= \frac{1}{\sqrt{x}}$ and the last inequality follows by the choice of $T_i$ where $S_{T_i -1} + \gamma_0 \leq \frac{c^2\sigma_{max}^2}{\beta^2 \Delta_i^2 \log K}$.\\
    
\textbf{Self bounding for C1}\\
    For  $C_1 - \alpha Reg_T$, let $b_t = \frac{c\sigma_t}{\alpha\sqrt{(S_t + \gamma_0)\log K}}$, then 
    \begin{align*}
        2C_1 - \alpha Reg_T &= \alpha \sum_{t=1}^{T} \sum_{i\neq i^*} \lr{- b_t x_{t,i} \log(x_{t,i}) - \Delta_i x_{t,i}}\\
        &\leq \alpha \sum_{t=1}^{T} \sum_{i\neq i^*}  \max_{z \in \RR}\lrc{- b_t z \log(z) - \Delta_i z}.
    \end{align*}
    Function $g(z) = - b_t z \log(z) - \Delta_i z$ is a concave function and maximum occurs when the derivative is zero. So we must have $- b_t \log(z) - b_t - \Delta_i = 0 \Rightarrow z = e^{-\frac{\Delta_i}{b_t} - 1}$ and substituting this gives us $\max_{z \in \RR} g(z) = b_t e^{-\frac{\Delta_i}{b_t} - 1}$. Now we can use it in the previous equation to upper bound $C_1 - \alpha Reg_T$ as the following
    \begin{align*}
        C_1 - \alpha Reg_T &\leq \alpha \sum_{t=1}^{T} \sum_{i\neq i^*} b_t e^{-\frac{\Delta_i}{b_t} - 1}\notag\\
        &= \sum_{i\neq i^*} \sum_{t=1}^{T} \frac{c\sigma_t}{\sqrt{(S_t + \gamma_0)\log K}} exp\lr{-\frac{\alpha\Delta_i \sqrt{(S_t + \gamma_0)\log K}}{c\sigma_t}-1}\notag\\
        &\leq \sum_{i\neq i^*} \sum_{t=1}^{T} \sigma_t \times \frac{c}{\sqrt{(S_t + \gamma_0)\log K}} exp\lr{-\frac{\alpha\Delta_i \sqrt{(S_t + \gamma_0)\log K}}{c\sigma_{max}}-1},
    \end{align*}
where $\sigma_{max} = \max_{t \in [T]}\{ \sigma_t\}$. Now let $g_i(x) = \frac{c}{\sqrt{x\log K}} exp\lr{-\frac{\alpha\Delta_i \sqrt{x\log K}}{c\sigma_{max}}-1}$ then for each $i$ we need to upper bound 
\[
\sum_{t=1}^{T} \sigma_t g_i(S_t) = \sum_{t=1}^{T} (S_t -S_{t-1}) g_i(S_t),
\]
which using Lemma \ref{lemma:integralinequality} can be upper bound by $\int_{S_0}^{S_T} g_i(x) dx$ because $g$ is nonincreasing. On the other hand for any $\delta, a \geq 0$ we know $\int  \frac{a}{\sqrt{x}} exp(-\frac{\delta \sqrt{x}}{a}-1) d x = -\frac{2a^2}{\delta}exp(-\frac{\delta \sqrt{x}}{a}-1)$. So using the closed form of $\int g_i(x) dx$ with denoting $\delta = \frac{\alpha \Delta_i}{\sigma_{max}}, a = \frac{c}{\sqrt{\log K}}$, we have 
    \begin{align}
         C_1 - \alpha Reg_T &\leq \sum_{i\neq i^*} \int_{S_0}^{S_T} g_i(x) dx \notag\\
        &= \sum_{i\neq i^*} \frac{-2c^2\sigma_{max}}{\alpha \Delta_i \log K}exp\lr{-\frac{\alpha\Delta_i \sqrt{x\log K}}{c\sigma_{max}}-1} |_{x = S_0}^{x = S_T}  \notag\\
        & = \frac{2c^2\sigma_{max}\lr{exp\lr{-\frac{\alpha\Delta_i \sqrt{S_0\log K}}{c\sigma_{max}}-1} - exp\lr{-\frac{\alpha\Delta_i \sqrt{S_T\log K}}{c\sigma_{max}}-1}}}{\alpha \Delta_i \log K}\notag\\
        &\leq \sum_{i\neq i^*} \frac{2c^2 \sigma_{max}}{\alpha\Delta_i \log K}\label{eq:c1arb},
    \end{align}
   Now merging \eqref{eq:c2arb} and \eqref{eq:c1arb} gives us
    \begin{align}
    C - Reg_T \leq \sum_{i\neq i^*} \frac{8c^2\sigma_{max}}{ \Delta_i \log K} \lr{\frac{1}{\beta} + \frac{1}{\alpha}} &= \sum_{i\neq i^*} \frac{8c^2\sigma_{max}}{ \Delta_i \log K} \lr{\frac{1}{1/4 - \alpha} + \frac{1}{\alpha}} \notag\\
    &\leq \sum_{i\neq i^*} \frac{128c^2\sigma_{max}}{\Delta_i \log K}\label{eq:selfCarb},
    \end{align}
    where the second inequality uses $\alpha = \frac{1}{8}$.

\end{proof}

  \subsection{Proof of the Stability Lemma}\label{sec:stabilitylemma}
The lemma has two parts, the first part is the general bound for the stability term and the second is the special case of that bound where we set $\alpha$ to some specific value to get a desirable bound.\\

Before starting the proof we provide one fact and one lemma that help us in the proof of the stability lemma.\\
\begin{fact}[\citep{zimmert2020}]\label{fact:incconvexf*'}
 $f_t^{*'}(x)$ is a convex monotonically increasing function.
 \begin{proof}
 The proof is available in Section 7.3 of the supplementary material of \citet{zimmert2020}.
 \end{proof}
\end{fact}

\begin{lemma}\label{lemma:stability}
Let $D_F(x,y) = F(x) - F(y) - \innerprod{x-y}{\nabla F(y)}$ be the Bergman divergence of a function $F$ as $f_t(x)$ defined in \eqref{eq:regularizer}, any $x \in \textbf{dom}(f_t)$, and any $\ell$ such that $\ell \geq -\gamma_t^{-1}$:
\[
D_{f_t^{*}}(f_t^{'}(x) - \ell, f_t^{'}(x)) \leq \frac{\ell^2}{2f_t^{''}(ex)}. 
\]
Moreover it is easy to see $f_t^{''}(ex)^{-1} \leq 4f_t^{''}(x)^{-1}$, that leads to have $D_{f_t^{*}}(f_t^{'}(x) - \ell, f_t^{'}(x)) \leq \frac{2\ell^2}{f_t^{''}(x)}$.
\end{lemma}
\begin{proof}

Taylor's theorem tells us that there exists $\xtil \in \lrs{f_t^{*'}(f_t^{'}(x)- \ell), f_t^{*'}(f_t^{'}(x))}$ such that
\[
 D_{f_t^{*}}(f_t^{'}(x) - \ell, f_t^{'}(x)) = \frac{1}{2}\ell^2 f_t^{*''}(f_t^{'}(\xtil)) = \frac{1}{2} \ell^2 f_t^{''}(\xtil)^{-1},
\]
where the second equality uses a direct property of convex conjugate operation. We have two cases for $\ell$:
\begin{enumerate}
    \item If $\ell \geq 0$ then based on Fact \ref{fact:incconvexf*'} we know $f_t^{*'}$ is increasing so  $\xtil \leq x$. On ther other hand $f^{''}(x)^{-1}$ is increasing so $f_t^{''}(\xtil)^{-1} \leq f_t^{''}(x)^{-1} \leq f_t^{''}(ex)^{-1}$.
    \item If $\ell < 0$ then $\xtil \in \lrs{f_t^{*'}(f_t^{'}(x)), f_t^{*'}(f_t^{'}(x)- \ell)}$. We will show that $f_t^{*'}(f_t^{'}(x)- \ell) \leq e x$, which by the choice of $\xtil$ in the specified interval we have  $\xtil \leq e x$ and consequently like the other case, we will end up having $f_t^{''}(\xtil)^{-1} \leq f_t^{''}(e x)^{-1}$.\\
Since $f^{*'}$ is increasing and $e x = f^{*'}(f^{'}(ex))$, then is suffices to prove $f^{'}(e x) \geq f^{'}(x) - \ell$ or equivalently $f^{'}(e x)  - f^{'}(x) \geq -\ell$.
So 
\begin{align*}
    f^{'}(e x) - f^{'}(x) &= \lr{-\eta_t^{-1} (e x)^{-1/2} + \gamma_t^{-1} \log(ex)} - \lr{-\eta_t^{-1} x^{-1/2} + \gamma_t^{-1} \log(x)} \\
    &= \eta_t^{-1} x^{-1/2} \lr{1 - \frac{1}{\sqrt{2}}} + \gamma_t^{-1} \geq \gamma_t^{-1}  \geq -\ell
\end{align*}
\end{enumerate} 
\end{proof}

\begin{proof}[\textbf{Proof of the First Part of the Stability Lemma}]
We know $x_{t} = \arg\min_{x\in\Delta^{K-1}}\langle\Lhat_t^{obs},x\rangle + F_t(x)$ so by the KKT conditions there exists $c \in \RR$ such that  $-\hat L_t^{obs} = \nabla F_t(x_t) - c_0 \textbf{1}_K$. On the other hands we know $\bar F_t(-L + c\textbf{1}_K) = \bar F_t(-L) + c$  for any $c \in \RR$ and $L \in \RR^K$ and the equality holds iff $c = 0$. Therefore, using these two facts we can rewrite the stability term as
\begin{align}
\sum_{t=1}^T {\innerprod{x_t}{\hat\ell_{t}^{obs}} + \bar F_t^*(-\hat{L}^{obs}_{t+1}) - \bar F_t^*(-\hat{L}^{obs}_{t})} &= \sum_{t=1}^T {\innerprod{x_t}{\hat\ell_{t}^{obs}-\alpha_t \textbf{1}_K} + \bar F_t^*(-\hat{L}^{obs}_{t+1} + (\alpha_t + c_0) \textbf{1}_K) - \bar F_t^*(-\hat{L}^{obs}_{t} + c_0\textbf{1}_K)}\notag\\
&= \sum_{t=1}^T {\innerprod{x_t}{\hat\ell_{t}^{obs}-\alpha_t \textbf{1}_K} + \bar F_t^*(\nabla F_t(x_t) - (\hat\ell_{t}^{obs}-\alpha_t \textbf{1}_K)) - \bar F_t^*(\nabla F_t(x_t))}\notag\\
&\leq \sum_{t=1}^T {\innerprod{x_t}{\hat\ell_{t}^{obs}-\alpha_t \textbf{1}_K} + F_t^*(\nabla F_t(x_t) - (\hat\ell_{t}^{obs}-\alpha_t \textbf{1}_K)) -  F_t^*(\nabla F_t(x_t))}\notag\\
&= \sum_{i = 1}^K D_{f_t^{*}}\lr{f_t^{'}(x_{t,i}) - (\hat\ell_{t,i}^{obs} - \alpha_t), f_t^{'}(x_{t,i})}\label{eq:stabilityanalysis},
\end{align}
where the inequality holds because $\bar F_t^*(L) \leq  F_t^*(L)$ for all $L \in \RR^K$ and the equality holds iff there exists $x$ such that  $L = \nabla F_t(x)$. Hence, since $\alpha_t \leq \gamma_t^{-1}$ then
$
\hat\ell_{t,i}^{obs}-\alpha_t \geq -\alpha_t \geq -\gamma_t^{-1}
$.
This implies that we can apply Lemma~\ref{lemma:stability} to get the following bound for \eqref{eq:stabilityanalysis}
\[
stability \leq \sum_{i=1}^K 2 f_t^{''}(x_{t,i})^{-1} (\hat\ell_{t,i}^{obs}-\alpha_t)^2.
\]
\end{proof}

 \begin{proof}[\textbf{Proof of the Second Part of the Stability Lemma}]
 First we must check whether $\alpha_t = \frac{\sum_{j=1}^K  f^{''}(x_{t,j})^{-1} \elltil_{t,j}}{\sum_{j=1}^K  f^{''}(x_{t,j})^{-1}}$ satisfies $\alpha_t \leq \gamma_t^{-1}$ or not:
\begin{align*}
    \alpha_t &= \frac{\sum_{j=1}^K  f^{''}(x_{t,j})^{-1} \elltil_{t,j}}{\sum_{j=1}^K  f^{''}(x_{t,j})^{-1}}\\
    &= \frac{\sum_{j=1}^K  f^{''}(x_{t,j})^{-1} \sum_{s\in A_t} \ellhat_{s,j}}{\sum_{j=1}^K  f^{''}(x_{t,j})^{-1}}\\
    &\leq 8 |A_t| (K-1)^\frac{1}{3} \leq 8 d_{max} (K-1)^\frac{1}{3} \leq \gamma_t^{-1},
\end{align*}
where the first inequality uses Lemma \ref{lemma:aux1}. To make the analysis looks simpler for all $i$ let define $z_{i} = f_t^{''}(x_{t,i})^{-1}$ then by substituting the value of $\alpha_t$ in stability expression we have
\begin{align}
    \sum_{i = 1}^K z_i ({\elltil}_{t,i}-\alpha_t)^2 &= \sum_{i = 1}^K z_i {\elltil}_{t,i}^2 - 2\sum_{i = 1}^K z_i {\elltil}_{t,i} \alpha_t + \sum_{i = 1}^K z_i \alpha_t^2\notag\\ 
    &= \sum_{i = 1}^K z_i {\elltil}_{t,i}^2 - \frac{(\sum_{i = 1}^K z_i {\elltil}_{t,i})^2}{\sum_{i = 1}^K z_i}\notag\\
    &= \sum_{i = 1}^K z_i {\elltil}_{t,i}^2 - \frac{\sum_{i = 1}^K z_i^2 {\elltil}_{t,i}^2}{\sum_{i = 1}^K z_i} - \frac{\sum_{i \neq j} z_i z_j {\elltil}_{t,i}{\elltil}_{t,j}}{\sum_{i = 1}^K z_i}\notag\\
    &=  \sum_{i = 1}^K \lr{z_i  - \frac{ z_i^2 }{\sum_{i = 1}^K z_i}}\lr{\sum_{s\in A_t} \ellhat_{s,i}}^2 - \frac{\sum_{i \neq j} z_i z_j \lr{\sum_{r,s \in A_t} {\ellhat}_{r,i}{\ellhat}_{s,j}}}{\sum_{i = 1}^K z_i}\notag\\
    &= \sum_{i = 1}^K \lr{z_i  - \frac{ z_i^2 }{\sum_{i = 1}^K z_i}}\lr{\sum_{s\in A_t} \ellhat_{s,i}^2}\label{eq:stab1}\\
    &+ \sum_{i = 1}^K \lr{ z_i  - \frac{z_i^2}{\sum_{i = 1}^K z_i}}\lr{\sum_{r,s \in A_t, r \neq s} \ellhat_{r,i}\ellhat_{s,i}} - \frac{\sum_{i \neq j} z_i z_j \lr{\sum_{r,s \in A_t} {\ellhat}_{s,i}{\ellhat}_{r,j}}}{\sum_{i = 1}^K z_i}\label{eq:stab2},
\end{align}
We call \eqref{eq:stab1} as Stab1 and \eqref{eq:stab2} as Stab2.\\
We start bounding the expectation of Stab1.
\begin{align}
    \EE[\text{Stab1}] = &\leq \EE\lrs{\sum_{i = 1}^K \lr{z_i  - \frac{ z_i^2}{\sum_{i = 1}^K z_i}}\lr{\sum_{s\in A_t} \ellhat_{s,i}^2}}\notag\\
    &= \EE\lrs{\sum_{i = 1}^K \lr{ z_i  - \frac{ z_i^2}{\sum_{i = 1}^K z_i}} \lr{\sum_{s\in A_t} \EE_s[\ellhat_{s,i}^2]}}\notag\\
    &= \EE\lrs{\sum_{i = 1}^K \lr{ z_i  - \frac{ z_i^2}{\sum_{i = 1}^K z_i}} \lr{\sum_{s\in A_t} \ell_{s,i}^2 x_{s,i}^{-1}}}\notag\\
    &\leq \sum_{s \in A_t}  \EE\lrs{ \sum_{i = 1}^K z_i x_{s,i}^{-1} - \frac{\sum_{i = 1}^K f_t^{''}(x_{t,i})^{-2}x_{s,i}^{-1}}{\sum_{i = 1}^K z_i}} \notag\\
    &\leq  \sum_{s \in A_t} \EE\lrs{ \sum_{i = 1}^K z_i x_{s,i}^{-1} (1 - x_{s,i}) }\notag\\
    &\leq  \sum_{s \in A_t} \EE\lrs{ \sum_{i = 1}^K 2\eta_t x_{t,i}^{3/2} x_{s,i}^{-1} (1 - x_{s,i}) }\label{eq:stab1bound},
\end{align}
where the second inequality bound losses by one and use the fact that $z_i  - \frac{z_i^2}{\sum_{i = 1}^K z_i} \geq 0$, the third inequality uses Cauchy-Schwarz inequality as   
$
\sum_{i = 1}^K z_i^2 x_{s,i}^{-1} = \lr{\sum_{i = 1}^K z_i^2 x_{s,i}^{-1}} \lr{\sum_{i =1}^K x_{s,i}} \geq \lr{\sum_{i = 1}^K z_i}^2
$, and the last inequality uses the fact that $z_i = f_t^{''}(x_{t,i})^{-1} \leq 2\eta_t x_{t,i}^{3/2}$. 


For the Stab2 we have
\begin{align}
    \EE[\text{Stab2}] &= \EE\lrs{\frac{1}{\sum_{i = 1}^K z_i}\lr{\sum_{i \neq j}\sum_{r,s \in A_t} -z_i z_j  {\ellhat}_{s,i}{\ellhat}_{r,j} + \sum_{i = 1}^K \sum_{r,s \in A_t, r \neq s} ( z_i(\sum_{i = 1}^K z_i) - z_i^2 ) \ellhat_{r,i}\ellhat_{s,i}}}\notag\\
    &= \EE\lrs{\frac{1}{\sum_{i = 1}^K z_i}\lr{\sum_{i \neq j}\sum_{r,s \in A_t} -z_i z_j  \mu_i \mu_j + \sum_{i = 1}^K \sum_{r,s \in A_t, r \neq s} ( z_i(\sum_{i = 1}^K z_i) - z_i^2 ) \mu_i^2}}\notag\\
     &= \EE\lrs{\frac{a_t(a_t-1)}{\sum_{i = 1}^K z_i}\lr{\sum_{i \neq j} -z_i z_j \mu_i \mu_j - \sum_{i = 1}^K z_i^2 \mu_i^2 + \sum_{i = 1}^K  z_i(\sum_{i = 1}^K z_i) \mu_i^2}}\notag\\
     &= \EE\lrs{\frac{a_t(a_t-1)}{\sum_{i = 1}^K z_i}\lr{-(\sum_{i=1}^K z_i\mu_i)^2 + (\sum_{i = 1}^K z_i \mu_i^2)(\sum_{i = 1}^K z_i) }}\notag\\
     &\leq \EE\lrs{\frac{a_t(a_t-1)}{\sum_{i = 1}^K z_i}\lr{-(\sum_{i = 1}^K z_i)^2\mu_{i^*}^2 + (\sum_{i = 1}^K z_i \mu_i^2)(\sum_{i = 1}^K z_i) }}\notag\\
     &= \EE\lrs{a_t(a_t-1)(\sum_{i = 1}^K z_i \mu_i^2 - \sum_{i=1}^K z_i\mu_{i^{*}}^2)}\notag\\
     &\leq \EE\lrs{a_t(a_t-1)(\sum_{i \neq i^*} 2z_i\Delta_i )}\notag\\
     &\leq \EE\lrs{\sum_{i \neq i^*} 2a_t(a_t-1)\gamma_t x_{t,i}\Delta_i }\label{eq:stab2bound}
    \end{align}
where the second equality follows by the fact that for all $s\in A_t$,  $x_{s}$ has no impact on determination of $x_t$ and for all different elements of $A_t$ such as $r, s \in A_t, r < s$, $x_r$ has no impact on determination of $x_s$. Regarding the inequalities, the first one holds because for all $i$ $\mu_i^* \leq \mu_i$, the second inequality follows by $\mu_i + \mu_{i^*} \leq 2$ and $\mu_i - \mu_{i^*} = \Delta_i$, and the last one substitutes $z_i = f^{''}(x_{t,i})^{-1} \leq \gamma_t x_{t,i}$.

Combining \eqref{eq:stab1bound} and \eqref{eq:stab2bound} completes the proof.
\end{proof}

 \section{Proof of the Key Lemma}\label{sec:proofkeylemma}
 \subsection{Auxiliary Materials for the Key Lemma}
 First we provide two facts and a lemma which are needed for proof of the key lemma.
  
\begin{fact}\label{fact:0}
 $f_t^{'}(x)$ is a concave function.
\begin{proof}
$f_t^{'}(x)= -\eta^{-1} x^{-1/2} + \gamma_t^{-1}\log x $ so the second derivative is $-\frac{3}{4}\eta^{-1} x^{-5/2} - \gamma_t^{-1}x^{-2} \leq 0$.
\end{proof}
\end{fact}
\begin{fact}\label{fact:derivative2ndinverse}
 $f_t^{''}(x)^{-1}$ is a convex function.
 \begin{proof}
 Let $g(x) = f_t^{''}(x)^{-1} = (\frac{\eta_t^{-1}x^{-3/2}}{2} + \gamma_t^{-1}x^{-1})^{-1}$, then the second derivative of $g(x)$ is 
 \[
g^{''}(x) =  \dfrac{{\eta}_\text{t}{\gamma}_\text{t}^2\cdot\left(2{\eta}_\text{t}x^\frac{7}{2}+3{\gamma}_\text{t}x^3\right)}{2\sqrt{x}\left(2{\eta}_\text{t}x^\frac{3}{2}+{\gamma}_\text{t}x\right)^3},
 \]
 which is positive.
 \end{proof}
\end{fact}
 
 \begin{lemma}\label{lemma:aux1}
 Assume that for $t$ and $s$ there exists $\alpha$ such that $x_{t,i} \leq \alpha x_{s,i}$  for all $i \in [K]$ and let $f(x) = \lr{-2\eta_t^{-1}\sqrt{x}+\gamma_t^{-1}x(\log x -1)}$, then we have the following inequality 
\[
\frac{\sum_{j=1}^K  f^{''}(x_{t,j})^{-1} \ellhat_{s,j}}{\sum_{j=1}^K  f^{''}(x_{t,j})^{-1}} \leq 2\alpha(K-1)^\frac{1}{3}.
\]
\end{lemma}
 \begin{proof}[Proof for Lemma~\ref{lemma:aux1}]

 Now we aim to bound  for any $s \in A$. 
 \begin{align}
     \frac{\sum_{i=1}^K  f^{''}(x_{t,i})^{-1} \ellhat_{s,i}}{\sum_{i=1}^K  f^{''}(x_{t,i})^{-1}} &= \frac{f^{''}(x_{t,i_s})^{-1} x_{s,i_s}^{-1} \ell_{s,i_s}}{\sum_{i=1}^K f^{''}(x_{t,i})^{-1}}\notag\\
     &\leq  \frac{f^{''}(x_{t,i_s})^{-1} x_{t,i_s}^{-1} \lr{x_{t,i_s}/x_{s,i_s}} }{\sum_{i=1}^K  f^{''}(x_{t,i})^{-1}}\notag\\
     &\leq \frac{f^{''}(x_{t,i_s})^{-1} \alpha x_{t,i_s}^{-1} }{\sum_{i=1}^K  f^{''}(x_{t,i})^{-1}}\notag\\
     &\leq \frac{\alpha f^{''}(x_{t,i_s})^{-1}  x_{t,i_s}^{-1} }{(K-1) f^{''}\lr{\frac{1-x_{t,i_s}}{K-1}}^{-1} + f^{''}(x_{t,i_s})^{-1}} ~~ \text{Define} ~ z := x_{t,i_s}\notag\\ 
     &=\frac{\alpha \lr{\eta_t^{-1} z^{-3/2} + 2\gamma_t^{-1}z^{-1}}^{-1} z^{-1}}{(K-1)\lr{\eta_t^{-1} (\frac{1-z}{K-1})^{-3/2} + 2\gamma_t^{-1}(\frac{1-z}{K-1})^{-1}}^{-1} + \lr{\eta_t^{-1} z^{-3/2} + 2\gamma_t^{-1}z^{-1}}^{-1}} \notag\\
     &= \alpha\lr{(1-z)\frac{\eta_t^{-1} z^{-1/2} + 2\gamma_t^{-1}}{\eta_t^{-1} \sqrt{K-1}(1-z)^{-1/2} + 2\gamma_t^{-1}} + z}^{-1}\label{eq:lossshiftineq}
 \end{align}
 where the first inequality follows by $\ell_{s,i_s} \leq 1$, the second one holds because of induction assumption that tells us for $s \leq t: t - s \leq d_{max} \Rightarrow x_{t,i} \leq \alpha x_{s,i}$, and the third inequality is due to convexity of $f^{''}(x)^{-1}$ from Fact \ref{fact:derivative2ndinverse}. Now for $z$ we have two cases, $z < \frac{1}{K}$ and $z \geq \frac{1}{K}$.
 
 \begin{itemize}
     \item[a)] $z \leq \frac{1}{K}$: This case implies
     \begin{align}
         \frac{1-z}{z} = \frac{1}{z} - 1 \geq K-1 &\Rightarrow (1-z)^{-1/2}\sqrt{K-1} \leq  z^{-1/2}\notag\\
         &\Rightarrow 1 \leq \frac{\eta_t^{-1} z^{-1/2} +  2\gamma_t^{-1}}{\eta_t^{-1} \sqrt{K-1}(1-z)^{-1/2} + 2\gamma_t^{-1}}\label{eq:case1shiftloss}\\
     \end{align}
     Plugging \eqref{eq:case1shiftloss} into \eqref{eq:lossshiftineq} gives us 
     \begin{equation*}
     \frac{\sum_{i=1}^K  f^{''}(x_{t,i})^{-1} \ellhat_{s,i}}{\sum_{i=1}^K  f^{''}(x_{t,i})^{-1}} \leq \alpha \lr{1-z + z}^{-1} = \alpha
    \end{equation*}
     \item[b)] $z \geq \frac{1}{K}$: Similar to previous case $z \geq \frac{1}{K}$ implies $\eta_t^{-1} z^{-1/2} \leq \eta_t^{-1} \sqrt{K-1}(1-z)^{-1/2}$ so the minimum of $\frac{\eta_t^{-1} z^{-1/2} +  2\gamma_t^{-1}}{\eta_t^{-1} \sqrt{K-1}(1-z)^{-1/2} + 2\gamma_t^{-1}}$ occurs when $2\gamma_t^{-1} = 0$. So substituting $2\gamma_t^{-1} = 0$ in \eqref{eq:lossshiftineq} leads us to have
     \begin{equation}\label{eq:lossshiftineq2}
         \frac{\sum_{i=1}^K  f^{''}(x_{t,i})^{-1} \ellhat_{s,i}}{\sum_{i=1}^K  f^{''}(x_{t,i})^{-1}} \leq \alpha \lr{(1-z)^{3/2}z^{-1/2} (K-1)^{-1/2} + z}^{-1}
     \end{equation}
     In this case again we have two following cases
     \begin{itemize}
         \item[b1)] $z \geq \frac{1}{(K-1)^{1/3} + 1}$: With this we have 
         \[
         \alpha \lr{(1-z)^{3/2}z^{-1/2} (K-1)^{-1/2} + z}^{-1} \leq \alpha z ^{-1} \leq \alpha \lr{(K-1)^{1/3} + 1} \leq 2\alpha (K-1)^{1/3}
         \]
         \item[b2)] $z \leq \frac{1}{(K-1)^{1/3} + 1}$: This tells us that $(1-z) \geq \frac{(K-1)^{1/3}}{(K-1)^{1/3} + 1} \geq \frac{1}{2} $ where we can use it in \eqref{eq:lossshiftineq2} as the following
         \begin{align*}
             \alpha \lr{(1-z)^{3/2}z^{-1/2} (K-1)^{-1/2} + z}^{-1}  &\leq \alpha \lr{\frac{z^{-1/2} (K-1)^{-1/2}}{\sqrt{8}} + z}^{-1}\\
             &= \alpha\lr{\frac{z^{-1/2} (K-1)^{-1/2}}{2\sqrt{8}} + \frac{z^{-1/2} (K-1)^{-1/2}}{2\sqrt{8}} + z}^{-1}\\
             &\leq \frac{\alpha}{3}\lr{\frac{(K-1)^{-1}}{32}}^{-1/3} \leq 2\alpha (K-1)^{1/3} 
         \end{align*}
         where the second inequality uses AM-GM inequality.
     \end{itemize}
 \end{itemize}
 
 So at the end combining results of all cases and setting $\alpha = 4$, gives us the upper bound $8(K-1)^{1/3}$.
 \end{proof}

 \subsection{Main Proof}
 \begin{proof}
The prove is based on induction on \emph{valid} pairs $(t,s)$, where we call a pair $(t,s)$ a valid pair if $s \leq t$ and $t - s \leq d_{max}$. The induction step for a valid pair $(t,s)$ requires induction assumption on two set of pairs. The first set includes the valid pairs $(t^{'},s^{'})$ such that $t^{'},s^{'} < t$ and the second set consists of the valid pairs $(t, s^{'})$ where $s < s^{'} \leq t$.  Hence, the induction base will be all pairs of $(t,t)$ for all $t \in [T]$ for which the induction statement trivially holds. Hence, it suffices to check the induction step for the valid pair $(t,s)$. 

As we mention in the proof sketch we have $x_t = \bar F_t^*(-\hat L_t^{obs})$ and $x_s = \bar F_s^*(-\hat L_s^{obs})$ and we introduce $\xtil = \bar F_s^*(-\hat L_t^{obs})$ as an auxiliary variable to bridge from $x_t$ and $x_{s}$. The bridge we use for relating $x_t$ to $x_s$ follows two following steps.
 
\textbf{Deviation Induced by the Loss Shift:} This step controls the drift when we fix the regularization (more precisely the learning rates) and shift the cumulative loss. We prove the following inequality:
    \[
        \xtil_{i} \leq \frac{3}{2}x_{s,i}.
    \]
    Note that this step uses the induction assumption for $(s,s-d_r)$ for all $r < s: r + d_r = s$.\\
\textbf{Deviation Induced by the Change of Regularizer:}
     In this step we bound the drift when the cumulative loss vector is fixed and we change the regularizer. We show that
     \[
        x_{t,i} \leq \frac{4}{3}\xtil_{i}.
     \]

 \subsubsection{Deviation Induced by the Change of Regularizer}\label{sec:regularizershift}
 
  The regularizer at any round $r$ is $F_{r}(x) = \sum_{i=1}^K f_r(x_i) = \sum_{i=1}^K \lr{-2\eta_{r}^{-1}\sqrt{x_i}+\gamma_r^{-1}x_i(\log x_i-1)}$. Since $x_t = \nabla\bar F_t^*(-\hat L_t^{obs})$ and $\xtil = \nabla\bar F_s^*(-\hat L_t^{obs})$, by the KKT conditions $\exists \mu, \mutil$ s.t. $\forall i$:
 \begin{align*}
      f_{s}^{'}(\xtil_i) &= -L_{s,i}^{obs} + \mu\\
      f_{t}^{'}(x_{t, i}) &= -L_{t,i}^{obs} + \mutil.
 \end{align*}
We also know that $\exists j: \xtil_j \geq x_{t,j}$ which leads to have
\begin{align*}
    -L_{t,j}^{obs} + \mutil = f_{t}^{'}(x_{t,j}) \leq f_{s}^{'}(x_{t,j}) \leq f_{s}^{'}(\xtil_j) = -L_{s,j}^{obs} + \mu,
\end{align*}
where the first inequality holds because the learning rates are decreasing and the second inequality is due to the fact that $f_{s}^{'}(x)$ is increasing. This implies that $\mutil \leq \mu$ which gives us the following inequality for all $i$:
\begin{align*}
    f_{t}^{'}(x_{t, i}) = -\frac{1}{\eta_{t}\sqrt{x_{t, i}}}+\frac{\log(x_{t, i})}{\gamma_{t}} \leq -\frac{1}{\eta_{s}\sqrt{\xtil_i}}+\frac{\log(\xtil_i)}{\gamma_{s}} = f_{s}^{'}(\xtil_i).
\end{align*}
Define $\alpha = x_{t,i}/\xtil_i$. So using above inequality we have
\begin{align*}
 \frac{1}{\eta_{s} \sqrt{\xtil_i}}-\frac{\log(\xtil_i)}{\gamma_{s}} &\leq \frac{1}{\eta_{t}\sqrt{ \alpha \xtil_i}}-\frac{\log(\xtil_i)}{\gamma_{t}} - \frac{\log(\alpha)}{\gamma_{t}}\notag ~~~~~~(\text{multiply both sides by $\eta_t \sqrt{\xtil_i}$ and rearrange}) \\
 \Rightarrow  \frac{1}{\sqrt{\alpha}} &\geq \frac{\eta_{t}}{\eta_{s}} +  2\sqrt{\xtil_i}\log(\sqrt{\xtil_i})\lr{\frac{\eta_{t}}{\gamma_{t}} - \frac{\eta_{t}}{\gamma_{s}}} + \log(\alpha) \frac{\eta_{t}}{\gamma_{t}}\sqrt{\xtil_i}\notag\\
 &\geq \frac{\eta_{t}}{\eta_{s}} + \min_{0\leq z \leq 1} \lrc{2 z\log(z)\lr{\frac{\eta_{t}}{\gamma_{t}} - \frac{\eta_{t}}{\gamma_{s}}} + \log(\alpha) \frac{\eta_{t}}{\gamma_{t}}z}\notag\\
 &\stackrel{(a)}{=} \frac{\eta_{t}}{\eta_{s}} - \frac{2}{e} \lr{\frac{\eta_{t}}{\gamma_{t}} - \frac{\eta_{t}}{\gamma_{s}}} \lr{\frac{1}{\sqrt{\alpha}}}^{\frac{\gamma_{t}^{-1}}{\gamma_{t}^{-1} - \gamma_{s}^{-1} }} \\
 &\stackrel{(b)}{\geq} \frac{\eta_{t}}{\eta_{s}} -  \lr{\frac{\eta_{t}}{\gamma_{t}} - \frac{\eta_{t}}{\gamma_{s}}} \frac{1}{\sqrt{\alpha}}
\end{align*}
where (a) holds because the subject function of the minimization problem is convex and equating the first derivative to zero gives $z = \lr{\frac{1}{\sqrt{\alpha}}}^{\frac{\gamma_{t}^{-1}}{\gamma_{t}^{-1} - \gamma_{s}^{-1} }}$ and (b) follows by ${\frac{\gamma_{t}^{-1}}{\gamma_{t}^{-1} - \gamma_{s}^{-1} }} \geq 1$ and $e \geq 2$. So rearranging the above result gives
 \begin{equation}\label{eq:alpha}
\alpha \leq \lr{\frac{\eta_s}{\gamma_{t}} - \frac{\eta_s}{\gamma_{s}} + \frac{\eta_{s}}{\eta_{t}}}^2 =  \lr{\eta_{s} (\gamma_t^{-1} - \gamma_s^{-1}) + \frac{\eta_s}{\eta_t}}^2.
 \end{equation}
Now we need to substitute the closed form of learning rates to obtain an upper bound for $\alpha$. As a reminder the learning rates are
\begin{align*}
    &\gamma_{s}^{-1} = \frac{1}{\sqrt{\log K}}\sqrt{\sum_{r=1}^s \sigma_r + \gamma_{0}},~ \eta_{s}^{-1} = \sqrt{s + \eta_{0}}\\
    &\gamma_{t}^{-1} = \frac{1}{\sqrt{\log K}}\sqrt{\sum_{r=1}^{s+d} \sigma_r + \gamma_{0}},~  \eta_{t}^{-1} = \sqrt{s + d + \eta_{0}},
\end{align*}
where $d = t-s$, $\eta_{0} = 10d_{max} + d_{max}^2/\lr{K^{1/3} \log(K)}^2 $, and $\gamma_{0} = 24^2d_{max}^2 K^{2/3} \log(K)$. Therefore in \eqref{eq:alpha} we have 
\begin{align}
    \eta_{s} \lr{\gamma_{t}^{-1} - \gamma_{s}^{-1}}&\leq \eta_{s}\frac{\sum_{r=s+1}^{s+d} \sigma_r}{\sqrt{\log(K) \lr{\sum_{r=1}^{s+d} \sigma_r + \gamma_{0}}}}\notag\\
    &\leq \eta_{s}\frac{\sum_{r=s+1}^{s+d} \sigma_r}{\sqrt{\log(K) \gamma_{0}}}\notag\\
    &\leq \frac{d_{max}^2}{\sqrt{\log(K)\gamma_{0} \eta_{0}}} \leq \frac{d_{max}^2}{\sqrt{24^2d_{max}^4}} = \frac{1}{24},\label{eq:ineq1}
\end{align}
where the third inequality follows by $d, \sigma_r \leq d_{max}$ for all $r$ and $\eta_{s} \leq \frac{1}{\sqrt{\eta_{0}}}$, and the last inequality holds because $\eta_{0} \geq 16 d_{max}^2/K^{2/3}$. On the other hand for $\frac{\eta_{s}}{\eta_{t}}$ in \eqref{eq:alpha} we have
\begin{align}
    \frac{\eta_{s}}{\eta_{t}} = \sqrt{\frac{s+ d +\eta_{0}}{s+\eta_{0}}} &= \sqrt{1 + \frac{d}{s + \eta_{0}}}\notag\\
    &\leq \sqrt{1 + \frac{d}{10d_{max}}}\notag\\
    &= \sqrt{\frac{10d_{max} + d}{10d_{max}}} \leq \sqrt{\frac{11}{10}}\label{eq:ineq2},
\end{align}
where the first and the second inequalities hold because $\eta_{0} \geq 2d_{max}$ and $d \leq d_{max}$, respectively.\\
Plugging \eqref{eq:ineq1} and \eqref{eq:ineq2} into \eqref{eq:alpha} gives us the following bound for $\alpha$.
\begin{equation}
\alpha \leq \lr{\sqrt{\frac{11}{10}} + \frac{1}{24}}^2 \leq \frac{4}{3}\label{eq:deviationregularizerresult}
\end{equation}




\subsubsection{Deviation Induced by the Loss Shift}\label{sec:lossshift}
We know $x_s = \nabla\bar F_{s}^*(-L_s^{obs})$ and $\xtil = \nabla\bar F_{s}^*(-L_t^{obs})$. Since the regularizer is fixed for $t$ and $s$ as $F_s(x) = \sum_{i=1}^K f_s(x_i)$ then for simplicity of the notation we drop $t$ and refer to $f_s(x)$ as $f(x)$. Let $\elltil = L_{t}^{obs} - L_s^{obs}$ then by the KKT conditions $\exists \mu, \mutil$ s.t. $\forall i$:
 \begin{align*}
      f^{'}(x_{s,i}) &= -L_{s,i} + \mu\\
      f^{'}(\xtil_{i}) &= -L_{t,i} + \mutil.
 \end{align*}
From the concavity of $f^{'}(x)$, derived from Fact \ref{fact:0}, we have 
\begin{align}
    (x_{s,i} - \xtil_{i})f^{''}(x_{s,i}) \leq f^{'}(x_{s,i}) - f^{'}(\xtil_{i}) \leq (x_{s,i} - \xtil_{i})f^{''}(\xtil_{i})\label{eq:concave}
\end{align}
 Using left side of \eqref{eq:concave} and the fact that $f^{''}(x_{s,i}) \geq 0$ gives us
 \begin{align}
     &x_{s,i} - \xtil_{i} \leq f^{''}(x_{s,i})^{-1}\lr{\mu - \mutil + \elltil_i} \Rightarrow\notag\\
     &\sum_{i=1}^K x_{s,i} - \xtil_{i} = 0 \leq \sum_{i=1}^K  f^{''}(x_{s,i})^{-1}\lr{\mu - \mutil + \elltil_i} \Rightarrow\notag\\
     & \mutil - \mu \leq \frac{\sum_{i=1}^K  f^{''}(x_{s,i})^{-1} \elltil_i}{\sum_{i=1}^K  f^{''}(x_{s,i})^{-1}}\label{eq:mubound}.
 \end{align}
Using the upper bound for $f^{'}(x_{s,i}) - f^{'}(\xtil_i)$ in \eqref{eq:concave} along with the upper bound for $\mutil - \mu$ and  the fact that $ f^{'}(x_{s,i}) - f^{'}(\xtil_i) = \mu - \mutil + \elltil_i$ result in
\begin{align}
    &(\xtil_{i} - x_{s,i}) f^{''}(\xtil_{i}) \leq \mutil - \mu - \elltil_i \leq \frac{\sum_{j=1}^K  f^{''}(x_{s,j})^{-1} \elltil_j}{\sum_{j=1}^K  f^{''}(x_{s,j})^{-1}}\Rightarrow\notag\\
    &\xtil_{i} \leq x_{s,i} + f^{''}(\xtil_{i})^{-1} \times \frac{\sum_{j=1}^K  f^{''}(x_{s,j})^{-1} \elltil_j}{\sum_{j=1}^K  f^{''}(x_{s,j})^{-1}}\\
    &\xtil_{i} \leq x_{s,i} +  \gamma_s \xtil_{i} \times \frac{\sum_{j=1}^K  f^{''}(x_{s,j})^{-1} \elltil_j}{\sum_{j=1}^K  f^{''}(x_{s,j})^{-1}}\label{eq:mainshift},
\end{align}
 where the last inequality holds because $f^{''}(\xtil_{i})^{-1} = \lr{\eta_s^{-1} \frac{1}{2}\xtil_{i}^{-3/2} + \gamma_s^{-1} \xtil_{i}^{-1}}^{-1}$. The next step in \eqref{eq:mainshift} is to bound $\frac{\sum_{j=1}^K  f^{''}(x_{s,j})^{-1} \elltil_j}{\sum_{j=1}^K  f^{''}(x_{s,j})^{-1}}$ to be able to give an concrete bound for $\xtil_i/x_{s,i}$. \\
 
 We know that $\elltil_i = \sum_{r \in A} \ellhat_{r,i}$ where $A = \lrc{r : s \leq r + d_r < t}$. If there exists $r \in A$ such that $r > s$ and $2x_{r,i} \leq x_{s,i}$, then combining it with the induction assumption on $(t,r), \ie~ $ $x_{t,i} \leq 2x_{r,i}$, leads to have $x_{t,i} \leq 2x_{r,i} \leq x_{s,i}$ which satisfies our desire. Otherwise, assume for all  $r \in A$ we have either $r \leq s$ or $x_{s,i} \leq 2x_{r,i}$. On the other hands, if $r \leq s$ we can use the induction assumption for $(s,r)$ that gives us  $x_{s,i} \leq 2x_{r,i}$. Consequently, we can assume that for all $s \in A$ inequality $x_{s,i} \leq 2x_{r,i}$ holds. Now using Lemma \ref{lemma:aux1} for any $r \in A$ we have 
 \begin{equation}\label{eq:inductionshift}
      \frac{\sum_{j=1}^K  f^{''}(x_{s,j})^{-1} \ellhat_{r,j}}{\sum_{j=1}^K  f^{''}(x_{s,j})^{-1}} \leq 4(K-1)^\frac{1}{3}.
 \end{equation}
 So summing up the above inequality for all $r \in A$ results in  $\frac{\sum_{j=1}^K  f^{''}(x_{s,j})^{-1} \elltil_j}{\sum_{j=1}^K  f^{''}(x_{s,j})^{-1}} \leq 4|A|(K-1)^\frac{1}{3}$. Now it suffices to inject this result to  \eqref{eq:mainshift}:
 \begin{align}
     \xtil_{i} &\leq x_{s,i} + 4 |A| \gamma_s \xtil_{i} (K-1)^\frac{1}{3} \Rightarrow\notag\\
      \xtil_{i} &\leq x_{s,i} \times \lr{\frac{1}{1 - 4 |A| \gamma_s (K-1)^{1/3}}}\label{eq:usefulldeviation}\\
     &\leq x_{s,i} \times \lr{\frac{1}{1 - 8\gamma_s d_{max} (K-1)^{1/3}}}\notag\\
     &\leq x_{s,i} \times \lr{\frac{1}{1 - 8\sqrt{\log K / \gamma_0} d_{max} (K-1)^{1/3}}}= \frac{x_{s,i}}{1 - 1/3} = \frac{3}{2} x_{s,i}\label{eq:lossdeviationresult},
 \end{align}
 where the third inequality uses $|A| \leq d_{max} + t - s \leq 2d_{max}$ and the last one uses the fact that $\gamma_s \leq \sqrt{\log(K)/\gamma_0}$  and $\gamma_0 = 24^2d_{max}^2 (K-1)^{2/3} \log(K)$.
 \\
 Combining \eqref{eq:lossshiftineq} and \eqref{eq:deviationregularizerresult} completes the proof.
\end{proof}

\section{Lower bounds}\label{sec:lowerboundsproof}
\begin{algorithm}
\caption{Adversarial choice of $\ell$}
\label{alg:adv ell}
\DontPrintSemicolon
\KwIn{$x$}
\KwInit{$\cI=\{\argmax_ix_i\}$}
\LinesNumberedHidden
\While{$\sum_{i\in \cI}x_i + \min_{i\in\bar\cI}x_i\leq \frac{2}{3}$ }{
	Update $\cI \leftarrow \cI\cup \{\argmin_{i\in\bar\cI}x_i\}$	\;
}
\Return{ $\ell_i =\begin{cases}\min\{1,\frac{\sum_{i\in\bar\cI}x_i}{\sum_{i\in\cI}x_i}\} &\text{ for }i\in \cI\\
\max\{-1,-\frac{\sum_{i\in\cI}x_i}{\sum_{i\in\bar\cI}x_i}\}&\text{ for }i\in \bar\cI\end{cases}$}
\end{algorithm}

\begin{lemma}
\label{lemma:variance of losses}
For any $x\in\Delta([K])$ such that $\max_ix_i\leq \frac{2}{3}$, the vector $\ell$ returned from Algorithm~\ref{alg:adv ell} satisfies
$\ell\in[-1,1]$, $\ip{x,\ell}=0$ and $\sum_{i=1}^K x_i\ell_i^2 \geq \frac{1}{2}$.
\end{lemma}
\begin{proof}
The first two properties follow directly from construction.
For the third property we bound the ratio of the two sets.
Assume $\sum_{i\in \cI}x_i<\frac{1}{3}$, then $\argmin_{i\in\bar\cI}x_i<\frac{1}{3}$ and the algorithm does not return yet.
The quantity in question is therefore bounded by
\begin{align*}
\sum_{i=1}^K x_i\ell_i^2 = \sum_{i\in\cI} x_i\ell_i^2+\sum_{i\in\bar\cI} x_i\ell_i^2 = p + (1-p)\left(\frac{p}{1-p}\right)^2 = \frac{p}{1-p}\geq \frac{1}{2}\,.
\end{align*}
\end{proof}
\begin{claim}
\label{clm:negentropy}
For the negentropy potential $F(x)=\eta^{-1}\sum_{i=1}^K\log(x_i)x_i$, it holds that
\begin{align*}
-\overline F^*(-L) -\min_i L_i = \eta^{-1}\log(\max_i \nabla \overline F^*(-L)_i)\,. 
\end{align*}
\end{claim}
\begin{proof}
Denote $i^*=\argmin_{i\in[K]}L_i$. It is well known that the exponential weights distribution is $(\nabla\overline F^*(-L))_i=\exp(-
\eta L_i)/(\sum_{j\in[K]})\exp(-
\eta L_j)$. Therefore the 
negentropy has an explicit form of the constrained convex conjugate: 
\begin{align*}
\overline F^*(-L) &= 
\ip{\nabla\overline F^*(-L),-L}-F(\nabla\overline F^*(-L))=\eta^{-1}\log(\sum_{i=1}^K\exp(-\eta L_i))
\,.
\end{align*}
Hence
\begin{align*}
-\overline F^*(-L) -L_{i^*} &=  -\eta^{-1}\log\left(\sum_{i=1}^K\exp(-\eta L_i)\right) + \eta^{-1}\log(\exp(-\eta L_{i^*}))\\
&=-\eta^{-1}\log\left(\frac{\sum_{i=1}^K\exp(-\eta L_i)}{\exp(-\eta L_{i^*})}\right)=\eta^{-1}\log\left(\nabla \overline F^*(-L)_{i^*}\right)
\end{align*}
\end{proof}

\begin{proof}[Proof of Theorem~\ref{thm:full-info var range}]
For ease of presentation, we will work with loss ranges $[-L_t/2,L_t/2]$, which is equivalent to loss ranges of $[0,L_t]$ in full-information games.
Assume that
\begin{align*}
    \frac{1}{2}\sum_{t=1}^{\lfloor\log_2(K)\rfloor}L_t \geq \frac{1}{32}\sqrt{\sum_{t=\lfloor\log_2(K)\rfloor}^T L_t^2\log(K) }\,.
\end{align*}
Define the active set $\cA_1 = [K]$.
At any time $t$, if $L_t$ is not among the largest loss ranges, we set $\ell_t$ to 0 and proceed with $\cA_{t+1}=\cA_{t}$.
Otherwise if $t\in\rho([\lfloor\log_2(K)\rfloor])$, we randomly select half of the arms in $\cA_t$ to assign $\ell_{t,i} = -L_t/2$ and the other half $\ell_{t,i} = L_t/2$. (In case of an uneven number $|\cA_t|$ we leave one arm at $0$.) All other losses are $0$. We reduce $\cA_{t+1}=\{i\in\cA_t\,|\,\ell_{t,i}<0\}$ to the set of arms that were negative.
The set $\cA_n$ will not be empty since we can repeat halving the action set exactly $\lfloor\log_2(K)\rfloor$ many times. The expected loss of any player is always 0, while the loss of the best arm is 
$\min_{a}\sum_{t=1}^T\ell_{t,a}=-\sum_{t=1}^{\lfloor\log_2(K)\rfloor}L_t/2$, hence
\begin{align*} 
\R^*\geq \sum_{t=1}^{\lfloor\log_2(K)\rfloor}L_t/2\,.
\end{align*}

It remains to show the case 
\begin{align*}
    \frac{1}{2}\sum_{t=1}^{\lfloor\log_2(K)\rfloor}L_t < \frac{1}{32}\sqrt{\sum_{t=\lfloor\log_2(K)\rfloor}^T L_t^2\log(K) }\,.
\end{align*}
In this case, note that we have
\begin{align}
    \sqrt{\sum_{t=\lfloor\log_2(K)\rfloor}^T L_t^2/\log(K) } > \frac{16}{\log(K)}\sum_{t=1}^{\lfloor\log_2(K)\rfloor}L_t > 16\frac{{\lfloor\log_2(K)\rfloor}}{\log(K)} L_{\lfloor\log_2(K)\rfloor}>8L_{\lfloor\log_2(K)\rfloor}\,.\label{eq:eta bound}
\end{align}
The high level idea is now to create a sequence of losses adapted to the choices of the algorithm.
Let $x_{ti} = \E[I_t=i|\ell_{t-1},\dots,\ell_{1}]$ be the expected trajectory of the algorithm
and let $z_{ti} = \exp(-\eta L_{ti})/\sum_{j=1}^K\exp(-\eta L_{tj})$ for
 $L_t = \sum_{s=1}^{t-1}\ell_t$ be the trajectory of EXP3.
 We show that it is possible to choose $\ell_t$ such that $0=\ip{z_t,\ell_t}\leq \ip{x_t,\ell_t}$, i.e. the regret of the algorithm cannot be smaller than that of EXP3. Finally we show that the construction of the losses ensures that the regret of EXP3 is lower bounded by the right quantity.
 Set $\eta = \sqrt{\log(K)/(\sum_{t=\lfloor\log_2(K)\rfloor}^TL_t^2)}$, and let the adversary follow Algorithm~\ref{alg:adversary} for the selection of the losses.
Denote $\tau:=\argmax\{t\in[T+1]\,|\,\ell_{t-1} \neq 0\}$, then the regret of algorithm $\cA$ can be bounded as
\begin{align*}
{Reg}_T(\cA) &= \sum_{t=1}^T \ip{x_t,\ell_t}-\min_{a^*\in\Delta([K])}\ip{a^*, L_{T+1}} \geq \sum_{t=1}^T \ip{z_t,\ell_t}-\min_{a^*\in\Delta([K])}\ip{a^*, L_{T+1}} = -\min_{a^*\in\Delta([K])}\ip{a^*, L_{T+1}}\,,
\end{align*}
where we use $\ip{x_t,\ell_t}\geq \ip{z_t,\ell_t}=0  $ by construction,
By expansion and Claim~\ref{clm:negentropy}, we have
\begin{align*}
    -\min_{a^*\in\Delta([K])}\ip{a^*, L_{T+1}}&=\eta^{-1}\log(K)-\eta^{-1}\log\left(\sum_{i=1}^K\exp(-\eta L_{T+1,i})\right)-\min_{a^*\in\Delta([K])}\ip{a^*, L_{T+1}} \\
    &\qquad+ \sum_{t=1}^T \eta^{-1}\log\left(\sum_{i=1}^K\exp(-\eta L_{t+1,i})\right)-\eta^{-1}\log\left(\sum_{i=1}^K\exp(-\eta L_{t,i})\right)\\
    &=\eta^{-1}\log(K)+\eta^{-1}\log(\max_{i\in[K]}z_{T+1,i})+\sum_{t=1}^T\eta^{-1}\log(\sum_{i=1}^Kz_{ti}\exp(-\eta\ell_{ti})
\end{align*}
The learning rate and setting the largest $\log_2(K)$ loss ranges to zero ensures that $|\eta\ell_{ti}|\leq \frac{1}{2}\eta L_{\lfloor\log_2(K)\rfloor}\leq \frac{1}{2}$.
Using that by Taylor's theorem and the monotonicity of the second derivative of $\exp$, we have for all $x\geq -\frac{1}{2}$: $\exp(x)\geq 1+x+\frac{1}{2}\exp''(-\frac{1}{2})x^2\geq 1+x+\frac{3}{10}x^2$, as well as by concavity of $\log$ for all $0\leq x \leq \frac{1}{4}$ we have $\log(1+x)\geq 4\log(5/4) x\geq \frac{5}{6}x$, we get for any $t\in[T]$ by Lemma~\ref{lemma:variance of losses}
\begin{align*}
    \eta^{-1}\log(\sum_{i=1}^Kz_{ti}\exp(-\eta\ell_{ti})\geq \eta^{-1}\log(1+\eta^2\frac{3}{10}\sum_{i=1}^Kz_{ti}\ell_{ti}^2)\geq \frac{\eta}{4}\sum_{i=1}^Kz_{ti}\ell_{ti}^2\leq \mathbb{I}\{\max_iz_{ti}\leq\frac{2}{3}\}\frac{\eta}{32}L_{\rho^{-1}(t)}^2\,.
\end{align*}
Now we have two possible event, either $\forall t\in[T]: \max_iz_{ti}\leq\frac{2}{3}$ and
\begin{align*}
    {Reg}_T(\cA)\geq \frac{\eta}{32}\sum_{t=\lfloor\log_2(K)\rfloor}L_t^2=\frac{1}{32}\sqrt{\sum_{t=\lfloor\log_2(K)\rfloor}L_t^2\log(K)}\,,
\end{align*}
or $\max_iz_{T+1,i}>\frac{2}{3}$ and
\begin{align*}
    {Reg}_T(\cA)\geq \eta^{-1}(\log(K)+\log(2/3))\geq \frac{1}{32}\eta^{-1}\log(K)=\frac{1}{32}\sqrt{\sum_{t=\lfloor\log_2(K)\rfloor}L_t^2\log(K)}\,.
\end{align*}
\end{proof}
\begin{algorithm}
\caption{Adversary}
\label{alg:adversary}
\KwIn{Actor $\cA$, learning rate $\eta$}
\DontPrintSemicolon
\LinesNumberedHidden
\For{$t= 1,\ldots,n$}{
	Set $\forall i: z_{ti}=\exp(-\eta L_{ti})/\sum_{j=1}^K\exp(-\eta L_{tj})$	\;
	\If{$\max_{i\in[K]}z_{ti}> \frac{2}{3}$ or $\rho(t)\leq \lfloor\log_2(K)\rfloor$}{$\ell_t=0$\;}
	\Else{
	Get $\ell$ from Algorithm~\ref{alg:adv ell} with $x=z_t$.\;
	Determine $x_t = \E[\cA((\ell_s)_{s=1}^{t-1})]$\;
	Set $\ell_t = \operatorname{sign}(\ip{x_t,\ell})L_{\rho^{-1}(t)}\ell/2$\;
}
}
\end{algorithm}


\newpage

\section{Detailed Regret Bound for Algorithm~\ref{alg:FTRL-delay} }\label{sec:analysistheorem1}
In this section we provide a detailed regret bound for Algorithm~\ref{alg:FTRL-delay}.\\
As we proved in Section~\ref{sec:proof} we have the following inequality for the drifted regret:
\begin{equation}\label{eq:driftedregretinequality}
    \Reg \leq 2\Regdrift + d_{max}
\end{equation}
So we start give bound for the drifted regret by following dividing the drifted regret to stability and penalty term as mentioned in Section~\ref{sec:proof}. Following the general analysis of the penalty term for FTRL \citet{abernethy2015} we have
\[
penalty \leq \sum_{t=2}^{T}\lr{ F_{t-1}(x_{t}) - F_{t}(x_{t})} +  F_T(x^*) - F_1(x_{1}),
\]
which gives us
\begin{align}
     penalty &= \sum_{t=2}^{T} \lr{2(\sum_{i=1}^K x_{t,i}^{\frac{1}{2}}-1) (\eta_t^{-1} - \eta_{t-1}^{-1}) - \sum_{i=1}^K x_{t,i}\log(x_{t,i}) (\gamma_t^{-1} - \gamma_{t-1}^{-1})} - 2\eta_1^{-1} + 2\sqrt{K}\eta_1^{-1} + \gamma_1^{-1}\log K \notag\\
    &\leq \sum_{t=2}^{T} \lr{2\sum_{i\neq i^*} x_{t,i}^{\frac{1}{2}} (\eta_t^{-1} - \eta_{t-1}^{-1}) - \sum_{i=1}^K x_{t,i}\log(x_{t,i}) (\gamma_t^{-1} - \gamma_{t-1}^{-1}) } + 2\sqrt{\eta_0(K-1)} + \sqrt{\gamma_0 \log K}\notag\\
    &\leq \sum_{t=2}^{T} \lr{2\sum_{i\neq i^*} \eta_{t} x_{t,i}^{\frac{1}{2}} - \sum_{i=1}^K  \frac{\sigma_t \gamma_t x_{t,i}\log(x_{t,i})}{\sqrt{\log K}} } + 2\sqrt{\eta_0(K-1)} + \sqrt{\gamma_0 \log K}\label{eq:penaltyfinalbound}
\end{align}
where the first inequality holds because $x_{t,i^*}^{\frac{1}{2}} \leq 1$ and the second inequality follows by $\eta_t^{-1} - \eta_{t-1}^{-1} = \sqrt{t+\eta_0} - \sqrt{t-1 + \eta_0} \leq \frac{1}{\sqrt{t+\eta_0}} = \eta_t$ and  $\gamma_t^{-1} - \gamma_{t-1}^{-1} = \frac{\gamma_t^{-2} - \gamma_{t-1}^{-2}}{\gamma_{t}^{-1} + \gamma_{t-1}^{-1}} \leq  \frac{\gamma_t^{-2} - \gamma_{t-1}^{-2}}{\gamma_{t}^{-1}} $.\\

For the stability term, we start from the bound given by Lemma~\ref{lemma:stabilityarb}:

\begin{align}
\EE[stability] &\leq \sum_{t=1}^T \sum_{i \neq i^*} 2\gamma_t (a_t-1)a_t \EE[x_{t,i}]\Delta_i  +  \sum_{t=1}^T \sum_{s \in A_t}  \sum_{i = 1}^K \eta_t \EE[x_{t,i}^{3/2} x_{s,i}^{-1} (1 - x_{s,i})]\label{eq:stabilitylemmabound}
\end{align}
In above inequality we know $a_t x_{t,i} = \sum_{s \in A_t} x_{t,i} $ and using Lemma~\ref{lemma:key} gives us $x_{t,i} \leq 2x_{s,i}$ for $s \in A_t$ then for the first part \eqref{eq:stabilitylemmabound}:
\begin{equation}\label{eq:stabilitypart1}
\sum_{t=1}^T \sum_{i \neq i^*} 2\gamma_t (a_t-1) a_t x_{t,i}\Delta_i \leq  \sum_{t=1}^T \sum_{i \neq i^*} \sum_{s \in A_t} 4\gamma_t (a_t-1)a_t x_{s,i}\Delta_i =  \sum_{t=1}^T \sum_{i \neq i^*} 4\gamma_{t+d_t} (a_{t+d_t}-1) x_{t,i}\Delta_i
\end{equation}
Furthermore, we can bound $x_{t,i}^{3/2} x_{s,i}^{-1} (1 - x_{s,i}) \leq 2^{3/2}x_{s,i}^{1/2} (1- x_{s,i})$. Moreover, in order to remove the best arm $i^*$ from the summation in the later bound we use $x_{t,i^*}^{3/2}x_{s,i^*}^{-1} (1- x_{s,i^*}) \leq 2 \sum_{i\neq i^*} x_{s,i} \leq \sum_{i\neq i^*} 2x_{s,i}^{1/2}$. So for the second part of  \eqref{eq:stabilitylemmabound} we have
\begin{align}
     \sum_{t=1}^T \sum_{s \in A_t}  \sum_{i = 1}^K \eta_t x_{t,i}^{3/2} x_{s,i}^{-1} (1 - x_{s,i}) &\leq \sum_{t=1}^T \sum_{s \in A_t}  \sum_{i = 1}^K \eta_t 2^{3/2}x_{s,i}^{1/2} (1 - x_{s,i})\notag\\
     &\leq   \sum_{t=1}^T \sum_{s \in A_t}  \sum_{i \neq i^*} \sqrt{8} \eta_t  x_{s,i}^{1/2} + \sum_{t=1}^T \sum_{s \in A_t} \sum_{i\neq i^*} 2\eta_t x_{s,i}^{1/2} \notag\\
     &\leq \sum_{t=1}^T   \sum_{i \neq i^*} 5 \eta_t   x_{t,i}^{1/2}\label{eq:stabilitypart2},
\end{align}
where the last inequality follows by the facts that we can change the order of the summations and that each $t$ belongs to exactly one $A_s$.
Plugging \eqref{eq:stabilitypart1} and \eqref{eq:stabilitypart2} into \eqref{eq:stabilitylemmabound} we have
\begin{equation}\label{eq:stabilityfinalbound}
    \EE[stability] \leq \EE\lrs{\sum_{t=1}^T \sum_{i \neq i^*} 4\gamma_{t+d_t} (a_{t+d_t}-1) x_{t,i}\Delta_i + \sum_{t=1}^T   \sum_{i \neq i^*} 5 \eta_t   x_{t,i}^{1/2}}.
\end{equation}
Now it suffices to combine \eqref{eq:stabilityfinalbound}, \eqref{eq:penaltyfinalbound} and \eqref{eq:driftedregretinequality} to get
\begin{align}\label{eq:concretemainbound}
 \Reg &\leq \EE\lrs{\underbrace{\sum_{t=1}^T \sum_{i \neq i^*} 14 \eta_t x_{t,i}^{1/2}}_{A} + \underbrace{\sum_{t=1}^T \sum_{i \neq i^*} 8\gamma_{t+d_t} (a_{t+d_t}-1) x_{t,i}\Delta_i}_B + \underbrace{\sum_{t=2}^{T}\sum_{i=1}^K  \frac{2\sigma_t \gamma_t x_{t,i}\log(1/x_{t,i})}{\log K}}_C}\notag\\
    &+  \underbrace{4\sqrt{\eta_0(K-1)} + 2\sqrt{\gamma_0 \log K} + d_{max}}_{D}.
\end{align}

 We rewrite the regret as 
 \[
 \Reg = 4\Reg -3\Reg \leq 4A - \Reg + 4B - \Reg + 4C - \Reg + 4D,
 \]
 in which by applying Lemmas~\ref{lemma:Abound},\ref{lemma:Bbound} and \ref{lemma:Cbound} we achieve 

\begin{align*}
    4A - \Reg &\leq  \sum_{i \neq i^*} \frac{56^2}{\Delta_i}\log (T/\eta_0 + 1)\\
    4B - \Reg &\leq 2\times32^2 a_{max}\log K\\
    4C - \Reg &\leq \sum_{i\neq i^*} \frac{16^2\sigma_{max}}{\Delta_i \log K}.
\end{align*}
Therefore the final regret bound is
\begin{align*}
\Reg &\leq \sum_{i \neq i^*} \frac{56^2}{\Delta_i}\log (T/\eta_0 + 1) + 2048 a_{max}\log K + \sum_{i\neq i^*} \frac{256\sigma_{max}}{\Delta_i \log K}\\
&+ 16\sqrt{\eta_0 (K-1)} + 8\sqrt{\gamma_0 \log K} + 4d_{max}.
\end{align*}

\section{Removing $1/\Delta_i$ from $\sigma_{max}/\Delta_i$ in the Regret Bound}\label{sec:removingdelta}
In this section we discuss how having an asymmetric learning rate $\gamma_{t,i} \simeq \gamma_t / \sqrt{\Delta_i}$ for negative entropy regularizer allows to remove the factor $\sum_{i\neq i^*} 1/\Delta_i$ in front of $\sigma_{max}$ in the regret bound.

In the analysis of Algorithm~\ref{alg:FTRL-delay} we divided the regret into stability and penalty expressions. Moreover, in each of the bounds for stability and penalty we have two terms which correspond to negative entropy and Tsallis parts of the hybrid regularizer. The terms related to negative entropy part in both stability and penalty bounds are
\[
\underbrace{\sum_{t=1}^T \sum_{i \neq i^*} \gamma_{t+d_t} (a_{t+d_t}-1) \EE[x_{t,i}]\Delta_i}_{B} + \underbrace{\sum_{i=1}^K \EE[x_{t,i}\log(1/x_{t,i})] (\gamma_t^{-1} - \gamma_{t-1}^{-1})}_{C},
\]
where $B$ and $C$, as we have seen in Section~\ref{sec:proof}, are due to stability and penalty terms,respectively. The idea here is to scale-up $\gamma_t$ to decrease $C$, however increasing $\gamma_t$ increases $B$. Hence, we are facing a trade off here. To deal with this trade-off we change the learning rates for negative entropy from symmetric $\gamma_{t}$ to asymmetric $\gamma_{t,i}$, and we expect this change only affect the parts of regret bound come from the negative entropy part of the regularizer, which are $B$ and $C$. This change results in to having two following terms instead,
\[
\underbrace{\sum_{t=1}^T \sum_{i \neq i^*} \gamma_{t+d_t,i} (a_{t+d_t}-1) \EE[x_{t,i}]\Delta_i}_{B_{new}} + \underbrace{\sum_{i=1}^K \EE[x_{t,i}\log(1/x_{t,i})] (\gamma_{t,i}^{-1} - \gamma_{t-1,i}^{-1})}_{C_{new}}.
\]
Here if we could choose $\gamma_{t,i} = \gamma_t / \sqrt{\Delta_i}$, then using the definition of $\gamma_t$ we would be able to rewrite  $B_{new}$ and $C_{new}$ as
\begin{align*}
B_{new} &= \Ocal \lr{\sum_{t=1}^T \sum_{i \neq i^*} \gamma_{t+d_t} (a_{t+d_t}-1) \EE[x_{t,i}]\sqrt{\Delta_i}}\\
C_{new} &= \Ocal \lr{\sum_{i=1}^K  \frac{\sigma_t \gamma_t \EE[x_{t,i}\log(1/x_{t,i})]\sqrt{\Delta_i}}{\sqrt{\log K}}}.
\end{align*}
Now we must see what is the result of applying the self-bounding technique on these new terms. For $B_{new}$ and $C_{new}$, following the similar analysis as Lemma~\ref{lemma:Bbound} and Lemma~\ref{lemma:Cbound} we can get 
\begin{align*}
4B_{new}  - \Reg &= \Ocal(a_{max} \log K) = \Ocal(d_{max} \log K) \\
4C_{new}  - \Reg &= \Ocal(\frac{\sigma_{max}}{\log K}).
\end{align*}
This implies that injecting $\sqrt{1/\Delta_i}$ in the negative entropy learning rates removes the factor $\sum_{i\neq i^*} \frac{1}{\Delta_i}$ in front of the $\sigma_{max}$. More interestingly this comes without having any significant changes in the other terms of regret bound.\\
As a result, we conjecture that replacing a good estimation of the suboptimal gaps namely $\hat \Delta_i$ in $\gamma_{t,i}$ as $\gamma_{t,i}= \gamma_t / \sqrt{\hat \Delta_i}$ might be also helpful to remove the multiplicative factors related to suboptimal gaps in front of the  $\sigma_{max}$. We leave this problem to future work. 

\end{document}